\documentclass{article}

\usepackage{microtype}
\usepackage{graphicx}
\usepackage{booktabs} 

\usepackage{hyperref}

\usepackage[accepted]{icml2023}

\usepackage{amsmath}
\usepackage{amssymb}
\usepackage{mathtools}
\usepackage{amsthm}
\usepackage{macros}
\usepackage[toc,page,header]{appendix}
\usepackage{minitoc}

\let\Contentsline\contentsline 
\renewcommand
\makeatletter
\renewcommand{\@dotsep}{10000} 
\makeatother

\doparttoc 
\faketableofcontents 

\usepackage[textsize=tiny]{todonotes}
\usepackage{bm}
\usepackage{tikz}
\usetikzlibrary{angles, quotes}
\usepackage{subcaption}
\usepackage{gensymb}
\usepackage{tikz-3dplot}
\usetikzlibrary{patterns}
\usepackage{pgfplots}
\usepackage{wrapfig}
\usepackage{sidecap}
\usepackage{enumitem}

\def\LinPhi{\phi}

\newcommand{\LinPhiStart}[1]{\phi_0}

\def\LinPar{w}

\newcommand{\LinParEmpNew}[1]{\widehat  w^{+}}

\newcommand{\CovEmpReg}[1]{\widehat \Sigma_{\lambda}}

\newcommand{\CovBootEmp}[1]{\widehat \Sigma_{\text{Boot}}}
\newcommand{\CovBoot}[1]{\Sigma_{\text{Boot}}}
\newcommand{\CovBootReg}[1]{\Sigma^{\reg}_{\text{Boot}}}

\def\OPC{\Concentrability^\star}

\def\reg{\ensuremath{\lambda}}

\newcommand{\LinExp}[1]{{\color{teal} \xi}}

\def\radius{r}

\def\Minimax{\mathcal M}
\def\MinimaxEmp{\widehat{\Minimax}}

\def\lTerm{\log(1/\FailureProbability)}

\def\PartSetTD{\widetilde{\mathcal U}}

\def\BeCom{\ensuremath{\beta}}
\newcommand{\gbest}[1]{\g_{#1}}
\newcommand{\gbestemp}[1]{\ghat_{#1}}

\def\Mag{\nu}
\def\factor{K}

\def\GlobalIBE{\mathcal I_{\fClass}}
\def\Incompleteness{\mathcal I}

\def\PE{\mathcal E}
\def\fhatFQ{\fhat_{\text{FQ}}}

\def\Backup{\mathrm T}
\def\crfc{\frac{\ln(\card{\fClass}/\FailureProbability)}{(1-\BeCom)\nSamples}}

\def\Increment{\Delta}
\newcommand{\CriticalCondition}[1]{
	\Rademacher_\nSamples 
	\Big\{
	\Cost(\f,\f) - \Cost(\gbest{\f},\f ) 
	\mid 
	\P \Uvar(\f) \leq 2#1^2
	\Big\}
	\lesssim #1^2
	\qquad \text{and} \qquad
	(\factor + 1) \frac{\log_2(1/(\FailureProbability\criticalradius)) }{\nSamples}  \lesssim #1^2.
}

\def\Event{\ensuremath{E}}
\def\EventY{\ensuremath{E}}
\def\PartEventY{\Event}

\icmltitlerunning{When is Realizability Sufficient 
for Off-Policy Reinforcement Learning?}

\begin{document}

\twocolumn[
\icmltitle{When is Realizability Sufficient 
for Off-Policy Reinforcement Learning?}



\icmlsetsymbol{equal}{*}

\begin{icmlauthorlist}
\icmlauthor{Andrea Zanette}{UCB}
\end{icmlauthorlist}

\icmlaffiliation{UCB}{Department of Electrical Engineering and Computer Sciences, University of California, Berkeley, United States of America}

\icmlcorrespondingauthor{Andrea Zanette}{zanette@berkeley.edu}

\icmlkeywords{Machine Learning, ICML}

\vskip 0.3in
]



\printAffiliationsAndNotice{}  

\begin{abstract}
Understanding when reinforcement learning 
algorithms can make successful off-policy
predictions---and when the may fail to do so--remains
an open problem.
Typically, model-free algorithms 
for reinforcement learning  
are analyzed under a condition called 
Bellman completeness when they operate
off-policy with function approximation,
unless additional conditions are met.
However, Bellman completeness 
is a requirement that is much stronger 
than realizability
and that is deemed to be too strong 
to hold in practice.
In this work, we relax this structural assumption 
and analyze the statistical complexity of off-policy 
reinforcement learning when only realizability holds 
for the prescribed function class.

We establish finite-sample guarantees for off-policy reinforcement learning 
that are free of the approximation error term known as inherent Bellman error,
and that depend on the interplay of three factors.
The first two are well known: they are the metric entropy of the function class
and the concentrability coefficient that represents the cost of learning off-policy.
The third factor is new, and it measures the violation of Bellman completeness,
namely the mis-alignment between the chosen function class 
and its image through the Bellman operator.
Our analysis directly applies to the solution found by 
temporal difference algorithms when they converge.
\end{abstract}

\section{Introduction}
Markov decision processes (MDP) \cite{puterman1994markov,Bertsekas_dyn1,Bertsekas_dyn2} 
provide a general framework for reinforcement learning (RL)
\cite{bertsekas1996neuro,sutton2018reinforcement},
which is a general paradigm for prediction and decision making under uncertainty.
Modern RL algorithms typically solve sequences of sub-problems 
that require estimating the value of a policy different from the one that generated the dataset,
a task broadly called \emph{off-policy} reinforcement learning. 
Moreover, function approximations are typically implemented to deal with large state-action spaces.

Various off-policy methods have been proposed, 
such as importance sampling \cite{precup2000eligibility,thomas2016data,jiang2016doubly} 
and weight learning \cite{uehara2020minimax,jiang2020minimax,zanette2022Bellman}.
Nonetheless, methods based on controlling the temporal difference error,
such as fitted Q iteration \cite{ernst2005tree,munos2008finite},
TD \cite{sutton1988learning}, 
and their variants such as $Q$-learning \cite{watkins1992q}, 
remain widely used especially with deep 
function approximation 
\cite{tesauro1995temporal,mnih2013playing,mnih2015human,mnih2016asynchronous, fujimoto2018addressing}.
We collectively refer to these algorithms as \emph{temporal difference (TD) methods}.

\paragraph{Bellman completeness: a fundamental RL notion}
When the state-action space is large,
TD methods are implemented with a function approximation class 
for the action value function. Their existing analyses 
\cite{munos2008finite,chen2019information,duan2020minimax,fan2020theoretical}
rely on a fundamental reinforcement learning notion 
known as \emph{Bellman completeness}, 
which must hold for these algorithms to succeed.
Completeness requires the chosen approximation space 
to fully capture each Bellman backup,
see \cref{fig:BellmanComplete}.
However, such requirement is deemed too strict to hold in practice.
What is more, even related algorithms that are theoretically more robust
than TD and fitted Q, such as the minimax variant 
\cite{antos2008learning},
also rely on Bellman completeness to properly function without approximation error. 

This led researchers to investigate fundamental limits 
\cite{chen2019information,zanette2020exponential,wang2020statistical,
weisz2020exponential,wang2021exponential,foster2021offline}.
Recently, \cite{foster2021offline} discovered that completeness 
is crucial in an information-theoretic sense:
even with seemingly benign distribution shifts,
exponential lower bounds quickly arise in the absence of Bellman completeness.

Unfortunately, completeness is a very hard condition to meet.
For example, when realizability is violated, the predictor class can be expanded 
so as to reduce the approximation error, 
a balancing act known as bias variance trade-off \cite{shalev2014understanding}.
On the contrary, when Bellman completeness is not satisfied,
enlarging the prescribed function class may make completeness even more violated,
because the Bellman backup of this new and bigger function class must now be correctly represented.

In summary, while realizability is also needed 
to make good predictions in Statistics, 
Bellman completeness seems like an additional requirement 
specific to reinforcement learning,
one that is intuitively very restrictive and undesirable, 
and unlikely to hold in practice, but seemingly necessary.

\paragraph{Contribution}
In this work we analyze the statistical complexity of off-policy reinforcement learning 
in settings where only realizability is assumed,
and bridge the gap between the Bellman complete case 
and the known exponential lower bounds that arise when 
Bellman completeness is ``extremely'' violated.
In order to characterize this intermediate regime,
we introduce the concept of local inherent Bellman errors
to measure the local violation of Bellman completeness.
We then establish off-policy error bounds 
for the solution found by 
the minimax reinforcement learning formulation \cite{antos2008learning},
first with function classes of finite-cardinality and then with more general, non-parametric ones.

Our error bounds depend on three critical factors:
1) the metric entropy of the chosen function class, 
2) a certain amplifying factor, called concentrability coefficient,
that arises due to the distribution shift, and 
3) a new amplifying factor that represents the mis-alignment between the prescribed
function class and its image through the Bellman operator. 
Furthermore, these error bounds apply to the widely used 
iterative TD methods when and if they do converge.

The main improvement compared to prior analyses is that
\emph{the violation of Bellman completeness is expressed as an amplifying factor
that affects the sample complexity, instead of as an approximation error term known 
as inherent Bellman error}.
The improvement arises from the application of a localization argument 
to measure the violation of Bellman completeness.
Effectively, this \textbf{removes the assumption of Bellman completeness for off-policy evaluation}:
instead, the lack of completeness is measured by a certain coefficient
---like the metric entropy measures the function capacity and the concentrability measures the distribution shift---that can 
in principle be computed.
We expect the insights of this paper to apply more broadly to other settings such as policy optimization or exploration.

Bellman complete models require all Bellman backups
to be contained in the prescribed function class.
In contrast, in our work the two are allowed to be only partially aligned.
It follows that the decision processes that can be studied with our framework are 
far richer and more realistic than those that are Bellman complete, 
because the image of the prescribed function class
through the Bellman operator can have a complex, truly high-dimensional structure.

Most literature is discussed in \cref{sec:Literature}.

\begin{figure*}
\centering
\begin{subfigure}{0.3\textwidth}
	\tdplotsetmaincoords{45}{30}
	\begin{tikzpicture}[scale=2.0, line cap=round, 
	line join=round,tdplot_main_coords]
	\def\origin{(0,0,0)}
	\def\Tf{(1,1,0.0)}
	\def\gf{(1,1,0)}
	\def\gfm{(0.5,0.5,0)}
	\def\F{(-0.6,-0.15,0)}
	\draw[color=gray!40, fill = gray!20] \gfm ellipse (1.5cm and 0.45cm);
	\draw[] \origin node[below]{$\f$}; 
	\draw[draw=none] (1,1,0.5) node[above]{\phantom{$\T\f$}};
	\draw[dashed ] \Tf node[right]{$\T\f$} -- \gf node[below right]{};
	\draw[densely dotted] \origin -- \gf;
	\draw [fill] \origin circle (0.5pt);
	\draw [fill] \Tf circle (0.5pt);
	\draw [fill] \gf circle (0.5pt);
	\draw plot[variable=\t,domain=0:1,samples=100,smooth]
  	({\t},{-0.5*\t + 1.5*\t^2},{0});
  	\draw[] \F node[above]{\textcolor{gray}{$\fClass$}};
\end{tikzpicture}
\caption{}
\label{fig:BellmanComplete}
\end{subfigure}
\qquad\qquad\qquad
\begin{subfigure}{0.3\textwidth}
	\tdplotsetmaincoords{45}{30}
	\begin{tikzpicture}[scale=2.0, line cap=round, 
	line join=round,tdplot_main_coords]
	\def\origin{(0,0,0)}
	\def\Tf{(1,1,0.5)}
	\def\gf{(1,1,0)}
	\def\gfm{(0.5,0.5,0)}
	\def\F{(-0.6,-0.15,0)}
	\draw[color=gray!40, fill = gray!20] \gfm ellipse (1.5cm and 0.45cm);
	\draw[] \origin node[below]{$\f$}; 
	\draw[dashed ] \Tf -- \gf node[below right]{$\Projector\T\f$};
	\draw[] \Tf node[above]{$\T\f$};
	\draw[densely dotted] \origin -- \gf;
	\draw [fill] \origin circle (0.5pt);
	\draw [fill] \Tf circle (0.5pt);
	\draw [fill] \gf circle (0.5pt);
	\draw plot[variable=\t,domain=0:1,samples=100,smooth]
  	({\t},{\t},{0.8*\t - 0.3*\t^2});
  	\draw[] \F node[above]{\textcolor{gray}{$\fClass$}};
\end{tikzpicture}
\caption{}
\label{fig:BellmanIncomplete}
\end{subfigure}
\caption{Bellman completeness (\cref{fig:BellmanComplete}) 
puts strong restrictions on the Bellman operator $\T$,
because the Bellman operator $\T$ must map the chosen function class $\fClass$ 
onto itself, i.e., $\T\fClass \subseteq \fClass$.
Without Bellman completeness (\cref{fig:BellmanIncomplete}), there is no restriction on $\T\fClass$,
although its alignment with $\fClass$ 
does influence the statistical complexity of off-policy reinforcement learning.}
\end{figure*}
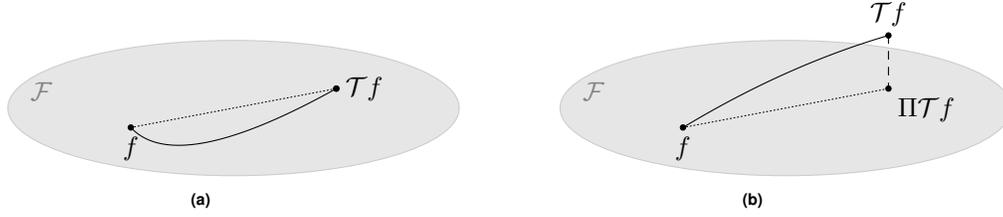

\section{Preliminaries}
\label{sec:Notation}
Here we recall the basic definitions; 
some additional background material 
can be found in \cref{app:Notation_Additional}.
\subsection{Notation and Set-up}
 
We focus on infinite-horizon discounted Markov decision
processes~\cite{puterman1994markov,bertsekas1996neuro,sutton2018reinforcement}
with discount factor $\discount \in [0,1)$, state space $\StateSpace$,
and an action set $\ActionSpace$.  For each state-action pair $\psa$,
there is a reward distribution $\RewardLaw\psa$ over $[0,1]$
with mean $\reward\psa$, and a transition function
$\TransitionLaw(\cdot \mid \state, \action)$.

A (stationary) target policy $\policy$ maps states to actions.  
Its action value function is denoted with $\fstar$.
It is defined as the discounted sum of future
rewards based on starting from the pair $\psa$, and then following the
policy $\policy$ in all future time steps
$\fstar\psa = \reward\psa + \sum_{\hstep = 1}^{\infty}
\discount^{\hstep} \E [ \reward_{\hstep}({\SState}_\hstep,
  {\AAction}_\hstep) \mid (\SState_0, \AAction_0) = \psa],
$
where the expectation is taken over trajectories with
$
\AAction_{\hstep} \sim \policy(\cdot \mid \SState_\hstep), \quad
\mbox{and} \quad \SState_{\hstep+1} \sim \TransitionLaw(\cdot \mid
\SState_\hstep, \AAction_\hstep) \quad \mbox{for $\hstep = 1, 2,
  \ldots$.}
$
We also use $\f(\state,\policy) = \E_{\AAction\sim \policy(\cdot \mid \state)} \f(\state, \AAction)$ 
 and define the
\emph{Bellman evaluation operator} and its empirical counterpart using the observed reward $\reward$ 
and successor state $\successorstate$ as
\begin{align*}
(\T\f)\psa & = \reward\psa + \discount \E_{\SState^+ \sim
    \TransitionLaw(\cdot \mid \state, \action)} \f(\SState^+ ,\policy), \\
	(\Backup \f) \pobs & = \reward + \discount \f(\successorstate,\policy). 
\end{align*}
The key property needed in our theorems is that $\Backup$ is a bounded operator.
The discounted occupancy measure of the target policy $\policy$ is given by
$
\dpi{\policy}\psa = (1 - \discount )
\sum_{\hstep=0}^\infty \discount^\hstep \Pro_\hstep[
  (\SState_\hstep,\AAction_\hstep) = \psa ],
$
where $\Pro_\hstep$ is the probability of encountering a certain state-action pair
when following $\policy$ from a given initial state.

We are interested in the \emph{prediction error} from a certain initial state $\state_0$,
which will be omitted later for brevity
$$
\PE(\f) = (\fstar - \f)(\state_0,\policy).
$$

Throughout the paper we assume that the learner has access to an action value function class $\fClass$ 
that contains the correct predictor.
\begin{assumption}[Realizability]
\label{asm:Realizability}
	$\fstar \in \fClass$.
\end{assumption}

\paragraph{Learning from a dataset}
We assume we have access to a dataset 
$\Dataset = \{ \sars{\iSample} \}_{\iSample= 1,\dots,\nSamples}$ 
that contains $\nSamples$ tuples. Each tuple contains a state $\state$,
an action $\action$, a reward $\reward$ and a successor state $\successorstate$.
In order to deal with the situation where the dataset is created using different policies,
we assume that the states and actions are sampled from an underlying distribution $\Dist$.
Conditioned on $\psa \sim \Dist$, the reward and successor state
in a certain tuple are sampled from the Markov reward process,
i.e., $\reward \sim \RewardLaw\psa$ and $\successorstate \sim \Transition{}\psa$.
The associated expectation operator over $\sars{}$ is often denoted with $\P$,
while its empirical counterpart over $\sars{} \in \Dataset$ is denoted with $\Pn$.

We commonly measure quantities using the norm induced by the distribution $\Dist$ and the policy $\policy$.
Let $\f$ be a function defined over the state-action space;
they are defined as
\begin{align*}
	\norm{\f}{\Dist}^2 = \E_{\psa \sim \Dist} [\f\psa]^2, \quad \norm{\f}{\policy}^2 = \E_{\psa \sim \dpi{\policy}} [\f\psa]^2.
\end{align*}

\paragraph{Projections}
The projection operator $\Projector$ onto $\fClass$  
takes in a function $\h$ and finds a function $\g \in \fClass$ closest to $\h$
\begin{align*}
	\Projector\h = \argmin_{\g \in \fClass} \norm{\g - \h}{\Dist}.
\end{align*}
In most cases we deal with, the function to project is the Bellman backup $\h = \T\f$,
and so it is convenient to denote 
the projected Bellman backup and the empirically projected backup with specific symbols, 
defined as 
\begin{align*}
	\gbest{\f} 
	& = \argmin_{\g\in\fClass}\norm{\g - \T\f}{\Dist}^2,
	\quad \text{and} \quad \\
	\numberthis{\label{eqn:BellmanProjectors}}
	\gEmp{\f} 
	& = \argmin_{\g\in\fClass}
	\frac{1}{\nSamples} \SumOverData{} 
	\Big(\g\psa - \Backup\f\pobs\Big)^2.
\end{align*}

\paragraph{Fitted Q}
\label{sec:AlgoAndLoss}
%
%
%

Fitted Q \cite{ernst2005tree} is a classical and well studied 
\cite{munos2005error,munos2008finite,chen2019information,duan2020minimax,fan2020theoretical} 
off-policy prediction and optimization algorithm.
In this paper we focus on the policy evaluation version of the algorithm, 
which starts from an initial iterate $\f_0 \in \fClass$ and updates it iteratively
by solving  
\begin{align*}
	\f_{k+1} & = \argmin_{\f \in \fClass} \frac{1}{\nSamples}  \SumOverData{} 
	\Big( \f \psa - \reward - \discount\f_k(\successorstate,\policy) \Big)^2.
\end{align*}
We indicate with $\fhatFQ$ the fixed point of fitted Q.

\subsection{Minimax Formulation and Inherent Bellman Error}
The fitted Q algorithm
is related to the minimax formulation
\cite{antos2008learning} in the sense that 
when fitted Q converges to a fixed point, 
such fixed point is a minimizer of the minimax formulation
\cite{chen2019information}.

\paragraph{Squared temporal difference cost}
Consider the following cost function,
which is the squared temporal difference error of the tuple 
$(\state,\action,\reward,\successorstate)$ 
evaluated using $\f$ as next-state value function 
and $\g$ as current function. It is defined as 
\begin{align}
\label{eqn:CostFunction}
\Cost(\g,\f) =  \Big(\g\psa - \reward - 
\discount \E_{\successoraction \sim \policy(\state)}\f(\successorstate, \successoraction) \Big)^2.
\end{align}
In order to find a predictor consistent with the dataset $\Dataset$,
one can try to minimize the empirical expectation of the above cost function
with $\g = \f$,
namely $\LossEmp(\f,\f)$ where\footnote{
It is useful to define the cost and its expectation by separating $\g$ and $\f$
in preparation for the discussion to follow.}
\begin{align*}
	\LossEmp(\g,\f) = \frac{1}{\card{\Dataset}}\SumOverData{} \Cost(\g,\f).
\end{align*}
Unfortunately, due to the double sampling issue \cite{baird1995residual,sutton2018reinforcement}, 
its expectation  
contains the bias term $\sigma(\f)^2$ 
(made explicit in \cref{lem:ExpectiCost}, but the fact is well known) 
representing the variance of the backup
\begin{align}
\label{eqn:ExpectiCost}
\E \Cost(\g,\f) 
= 
\norm{\g - \T\f}{\Dist}^2
+ \sigma^2(\f).
\end{align}
The variance term $\sigma(\f)^2$ arises even when $\g = \f$ in the cost function.
This implies that in the limit of infinite data the minimizer of $\E\Cost(\f,\f)$
must trade-off minimizing the mean-squared Bellman error $\norm{\f - \T\f}{\Dist}^2$ 
with minimizing the variance $\sigma(\f)^2$ of the backup.
The resulting procedure may converge to a solution different from the optimal predictor $\fstar$
even in the realizable setting.

\paragraph{A different cost function}
To remedy this issue\footnote{In practice iterative algorithms are used,
but the algorithm studied here is closely related to the iterative TD algorithms.}, 
the following cost function was introduced in \cite{antos2008learning}:
\begin{align*}
	\Cost(\f, \f) - \Cost(\g, \f).
\end{align*}
Compared to the squared TD cost function in \cref{eqn:CostFunction}, 
which would be minimized with $\g = \f$,
the modified cost function contains the correction term $- \Cost(\g, \f)$.
The expectation of the correction term generates the conditional variance of the backup $\sigma(\f)^2$
which then cancels the one present in $\E\Cost(\f,\f)$. 
We have $	\E [\Cost(\f, \f) - \Cost(\g, \f)] =$
\begin{align*}
	& =
	\norm{\f - \T\f}{\Dist}^2 + 
	\sigma(\f)^2 - \norm{\g - \T\f}{\Dist}^2
	- \sigma(\f)^2 \\
	& =
	\norm{\f - \T\f}{\Dist}^2 - \norm{\g - \T\f}{\Dist}^2.
	\numberthis{\label{eqn:ExpectedCostModified}}
\end{align*}
While the modified cost function is successful in cancelling the unwanted term $\sigma(\f)^2$,
it has introduced a different bias term represented by $\norm{\g - \T\f}{\Dist}^2$. 
In order to keep this bias at a minimum,  
the function $\g$ should be selected so as to minimize it, ideally as
\begin{align*}
	\gbest{\f} = \min_{\g\in\fClass} \norm{\g - \T\f}{\Dist}^2.
\end{align*} 
The population-level loss to minimize is \cite{antos2008learning}
\begin{align}
\label{eqn:minimax}
	\Minimax(\f)
	=
	\norm{\f - \T\f}{\Dist}^2 - \min_{\g\in\fClass}\norm{\g - \T\f}{\Dist}^2.
\end{align}

The resulting empirical program to minimize over $\f$ is
\begin{align*}
	\MinimaxEmp(\f) 
	& = 
	\LossEmp(\f,\f) - \min_{\g \in \fClass} \LossEmp(\g,\f)
\end{align*}
Its empirical minimizer $\fhat$ is of interest to us:
\begin{align*}
	\fhat \in \argmin_{\f \in \fClass} \MinimaxEmp(\f).
\end{align*} 
The fact that the fitted Q fixed point minimizes $\MinimaxEmp(\f)$
(see e.g., \cite{chen2019information}) motivates the study of the minimax formulation. 

\paragraph{Completeness removes the bias}
Despite the above effort to reduce the bias term,
the term $\inf_{\g\in\fClass}\norm{\g - \T\f}{\Dist}^2$ 
still affects the estimation quality of the mean-squared Bellman error,
and it is unclear whether that is better than $\sigma(\f)^2$. 
A notable case where such correction is desirable is when the bias term 
$\min_{\g\in\fClass}\norm{\g - \T\f}{\Dist}^2$ is zero for all $\f \in \fClass$, 
a condition called Bellman completeness. 
In this case, the population-level loss $\Minimax(\f)$ coincides with the mean-squared Bellman error,
i.e., under Bellman completeness we have
\begin{align}
\label{eqn:MinComp}
	\Minimax(\f) = \norm{\f - \T\f}{\Dist}^2.
\end{align}
Therefore, minimizing $\Minimax$ directly minimizes the mean-squared Bellman error.

\paragraph{Inherent Bellman errors}
When completeness starts to be violated, only part of $\T\f$ is `captured' by $\fClass$,
and an angle between the two arises, see  \cref{fig:BellmanIncomplete}.
Although in this cases the backup $\T\f$ is not contained in $\fClass$, 
we can still consider its projection onto $\fClass$ defined in \cref{eqn:BellmanProjectors}.
As the projection discards potentially useful informations about the backup $\T\f$,
we expect an error to arise. Such error is
the component of the backup $\T\f$ not captured by $\fClass$:
\begin{align}
\label{eqn:residual}
\inf_{\g \in \fClass}\norm{\g - \T\f}{\Dist}.	
\end{align}
An algorithm like fitted Q typically considers different functions $\f \in \fClass$
through its execution, and the projection error is propagated through the iterations. 
Moreover, such error term is present in the definition of the minimax program in \cref{eqn:minimax},
and so its presence seems to be unavoidable.
Generally, a \emph{worst-case} analysis is adopted, and
the worst-case value of the residual over $\f \in \fClass$ is called inherent Bellman error
of the function class $\fClass$ 
\begin{align}
	\label{eqn:IBE}
	\GlobalIBE
	=
	\sup_{\f \in \fClass} \inf_{\g\in\fClass} \norm{\g - \T\f}{\Dist}.
\end{align}
(Other definitions based on different norms are possible).
The inherent Bellman error is zero for the Bellman complete case in \cref{fig:BellmanComplete}; 
the less the Bellman backup is aligned with $\fClass$
the bigger it becomes (cfr. \cref{fig:BellmanIncomplete}).

%

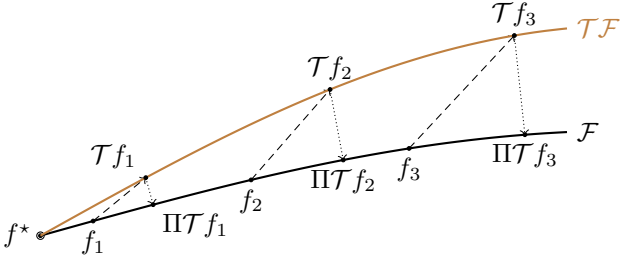
\begin{SCfigure*}
\begin{tikzpicture}[hinge/.style = {fill=white, draw=black}, scale=0.70]
	\def\cOne{2}
	\def\cTwo{4}
	\def\sOne{8}
	\def\sTwo{8}
	\newcommand{\yOne}[1]{\cOne*sin(\sOne*#1)}
	\newcommand{\yTwo}[1]{\cTwo*sin(\sTwo*#1)}
  \coordinate (fstar) at (0,0);
  \coordinate (f1)    at (1,{\yOne{1}});
  \coordinate (Tf1)   at (2,{\yTwo{2}});
  \coordinate (gf1)   at (2.14,{\yOne{2.14}});
  \coordinate (f2)    at (4,{\yOne{4}});
  \coordinate (Tf2)   at (5.5,{\yTwo{5.5}});
  \coordinate (gf2)   at (5.75,{\yOne{5.75}});
  \coordinate (f3)    at (7,{\yOne{7}});
  \coordinate (gf3)   at (9.2,{\yOne{9.2}});
  \coordinate (Tf3)   at (9,{\yTwo{9}});
   \draw[->, densely dashed ] (f1)node[below]{$\f_1$}--(Tf1)node[above left]{$\T\f_1$};
  \draw[->, densely dashed ] (f2)node[below]{$\f_2$}--(Tf2)node[above]{$\T\f_2$};
  \draw[->, densely dashed ] (f3)node[below]{$\f_3$}--(Tf3)node[above]{$\T\f_3$};
  \draw[->, densely dotted ] (Tf1)--(gf1)node[below right]{$\Projector\T\f_1$};
  \draw[->, densely dotted ] (Tf2)--(gf2)node[below]{$\Projector\T\f_2$};
  \draw[->, densely dotted ] (Tf3)--(gf3)node[below]{$\Projector\T\f_3$};
  \draw[hinge] (0,0)circle(2pt);
  \draw[fill=black] (0,0)circle(1pt)node[left]{$\fstar$};
  \draw[thick,black,domain=0:10,smooth] plot (\x,{\yOne{\x}})node[right]{$\fClass$};
  \draw[thick,brown,domain=0:10,smooth] plot (\x,{\yTwo{\x}})node[right]{$\T\fClass$};
  \draw[black, fill = black] (f1)circle(1pt);
  \draw[black, fill = black] (f2)circle(1pt);
  \draw[black, fill = black] (f3)circle(1pt);
  \draw[black, fill = black] (Tf1)circle(1pt);
  \draw[black, fill = black] (Tf2)circle(1pt);
  \draw[black, fill = black] (Tf3)circle(1pt);
  \draw[black, fill = black] (gf1)circle(1pt);
  \draw[black, fill = black] (gf2)circle(1pt);
  \draw[black, fill = black] (gf3)circle(1pt);
\end{tikzpicture}
\caption{Local inherent Bellman errors. 
The norm of the un-captured component of the Bellman error $\Projector\T\f - \T\f$,
when maximized over $\f \in \fClass$,
is the inherent Bellman error.
For every function $\f \in \fClass$, such un-captured component 
is always a fraction of the Bellman error $\f-\T\f$.
When the Bellman error is reduced, 
its un-captured component also gets reduced.
This means that the `effective' inherent Bellman error seen by an algorithm decreases 
as the algorithm approaches the optimal predictor $\fstar$ along $\fClass$.
In order to leverage this observation in the analysis,
we \emph{localize} the inherent Bellman error
to a subset of functions where
the empirical predictor $\fhat$ returned by the minimax algorithm is expected to be.
In this way, we can replace the inherent Bellman error in \cref{eqn:IBE},
which is defined globally over $\fClass$,
with a more localized version defined over
a smaller class $\widetilde \fClass \subset\fClass$ that contains $\fhat$.}
\label{fig:explain}
\end{SCfigure*}

\section{Local Inherent Bellman Errors}
\label{sec:Localization}
In this section we introduce the core concept of this paper,
namely the local inherent Bellman errors and the related notion of $\BeCom$-incompleteness;
they are needed to convey the main message of the paper when Bellman completeness is violated. 
From a technical standpoint,
our development is inspired by the localization argument of \cite{bartlett2005local},
which is a now a standard tool in statistics 
to obtain fast regression rates \cite{wainwright2019high}.
Our use of localization, however, concerns a different quantity---the inherent Bellman error---and
brings an even more consequential improvement, i.e., that of removing the approximation error
term connected to the lack of Bellman completeness.

Some intuition is provided in \cref{fig:explain}, while the definitions are motivated as follows.
If Bellman completeness was satisfied then minimizing $\Minimax$
would directly minimize the mean-squared Bellman error,
see \cref{eqn:MinComp}.
When completeness is violated, our hope is that the mean-squared Bellman error is still minimized
by the minimax algorithm.
In other words, we hope that $\fhat$ enjoys small mean-squared Bellman error
$ 
	\norm{\fhat - \T\fhat}{\Dist}^2
$.
If that is the case, $\fhat$ must belong to the set of predictors $\fClass(\radius)$
whose Bellman error is, say, at most $\radius$ for some positive value $\radius$:
\begin{align*}
	\fClass(\radius) = \{ \f \in \fClass \mid \norm{\f - \T\f}{\Dist} \leq \radius \}.
\end{align*} 
If $\fhat$ is known to belong to the set $\fClass(\radius)$, 
the inherent Bellman error that should arise in a performance bound 
is one where the predictor $\f$ is restricted to the class $\fClass(\radius)$.
The value of the inherent Bellman error constructed in this way as a function of $\radius$
is what we call \emph{incompleteness function}.
\begin{definition}[Incompleteness Function]
	The incompleteness function $\Incompleteness$ 
	(or localized inherent Bellman error) is the function
	\begin{align*}
	\Incompleteness(\radius)
	=
	\sup_{\f\in\fClass(\radius)}\inf_{\g \in \fClass}\norm{\g - \T\f}{\Dist}
	.
\end{align*}
\end{definition}
In other words, the incompleteness function is the inherent Bellman error
\emph{localized} to the set of functions of small mean-squared Bellman error $\norm{\f-\T\f}{\Dist}$.
When $\radius \rightarrow \infty$, the localized inherent Bellman error
recovers the inherent Bellman error, i.e.,  $\Incompleteness(\infty) = \GlobalIBE$.
Notice that if the model is misspecified ($\fstar \not \in \fClass$)
then the set $\fClass(\radius)$ may be empty
for small values of $\radius$, and so the incompleteness function 
is defined only up to a certain value of $\radius$.
 
To summarize, our expectation is that the empirical solution $\fhat$ belongs to $\fClass(\radius)$
for an appropriate value of $\radius$.
In that case, the inherent Bellman error `felt' by the minimax algorithm
should be $\Incompleteness(\radius)$.
When $\radius$ decreases, the function $\Incompleteness(\radius)$ should also decrease 
because it is an error associated to a smaller set.
This intuition on the behavior of 
the local inherent Bellman errors is correct, 
and it is formalized by the following proposition,
which is proved in \cref{app:iBehav}.
\begin{proposition}[Behavior of Local Inherent Bellman Errors] 
\label{prop:iBehav}
The following holds true:
	\begin{itemize}
		\item $\Incompleteness(\radius)$ is increasing with $\radius$;
		\item if realizability holds then $\Incompleteness(0) = 0$.
	\end{itemize}
\end{proposition} 
\cref{fig:Complete,fig:Linear,fig:Quadratic} illustrate 
possible shapes for the incompleteness function
in the realizable case, while \cref{fig:Misspecified} shows one where realizability
is violated (i.e., when $\fstar \not \in \fClass$).

In the sequel we focus on the realizable case to make the analysis clearer, i.e., on function classes that satisfy \cref{asm:Realizability}.
Although in this case the local inherent Bellman error always converges to zero,
it might do so at different speeds. 
The average rate of convergence to zero 
is denoted with $\BeCom$ and it determines the problem complexity.

\newcommand{\PlotIncompletenessFunction}[3]
{
    \begin{tikzpicture}
    \def\scale{3.5}
    \def\ExtFactor{0.1}
        \draw[->] (0,0) -- (\scale,0) node[below] {$\radius$};
        \draw[->] (0,0) -- (0,0.5*\scale) node[left] {$\Incompleteness(\radius)$};
        \draw[thick,purple,domain=#3 :\scale,smooth] plot (\x,{#1});
        #2 
    \end{tikzpicture}
}

\begin{figure*}
    \begin{subfigure}[b]{0.18\textwidth}
    	\PlotIncompletenessFunction{0*\x}{}{0}
        \caption{}
        \label{fig:Complete}
    \end{subfigure}
    \qquad %
    \begin{subfigure}[b]{0.18\textwidth}
    	\PlotIncompletenessFunction{0.3*\x}{}{0}
        \caption{}
        \label{fig:Linear}
    \end{subfigure}
    \qquad %
    ~ 
    \begin{subfigure}[b]{0.18\textwidth}
    	\PlotIncompletenessFunction{0.1*\x^2}
    	{\draw[dashed, brown] (0,0) -- ({\scale},{0.1*(\scale)^2}) node[midway, right, above, sloped]{$\BeCom\radius$};}{0}
        \caption{}
        \label{fig:Quadratic}
    \end{subfigure}
    \qquad
    \begin{subfigure}[b]{0.18\textwidth}
    	\PlotIncompletenessFunction{0.3 + 0.3*sqrt(\x)}{}{0.2}
        \caption{}
        \label{fig:Misspecified}
    \end{subfigure}
    \caption{Stylized representations of possible shapes of $\Incompleteness$}
    \label{fig:iFig}
\end{figure*}

\subsection{\texorpdfstring{$\BeCom$}{beta}-incomplete MDPs}
Let us gain some intuition by considering a linear problem,
namely one where the function class $\fClass$ is linear. 
It is defined by a feature extractor $\LinPhi$ that maps state-action pairs 
to real vectors in $\Reals^\dim$, 
as $\fClassLin = \{ \LinPhi^\top \LinPar \mid \LinPar \in \Reals^\dim \}$.

When the class is linear and realizability holds,
the localized inherent Bellman error $\Incompleteness(\cdot)$ 
always increases at a linear rate, a fact that we verify in \cref{sec:iLin}.
\begin{proposition}[Linearly Incomplete MDPs]
\label{prop:iLin}
If $\fClass = \fClassLin$ then $\Incompleteness(\radius) = \BeCom\radius$ for all $\radius\geq 0$.
\end{proposition}

In this case, we say that the system is $\BeCom$-incomplete.
When $\BeCom = 0$, the MDP is linear Bellman complete 
\cite{zanette2020learning,duan2020minimax} and that 
corresponds to the situation in \cref{fig:Complete}.
On the contrary, the higher $\BeCom$ is, and the farther from $\fstar$
(i.e., the higher the radius $\radius$),
the more Bellman completeness is violated,
a situation in display in \cref{fig:Linear}. 

When $\fClass$ is non-linear we expect the local inherent Bellman error $\Incompleteness$ 
to exhibit a more complex behavior. 
It must still comply with \cref{prop:iBehav}, 
namely it must start from zero and increase as the radius increases. 
In these cases it is a good idea to define a quantity to 
capture its global behavior.
Such quantity should put a bound on the average rate of increase of $\Incompleteness$,
i.e., such that 
\begin{align}
\label{eqn:BoundedIncrease}
	\Incompleteness(\radius) \leq \BeCom\radius.
\end{align}
With this goal in mind, we give the following definition for $\BeCom$,
one that applies to the linear and the non-linear setting.
\begin{definition}[$\BeCom$-incompleteness]
\label{def:BeCom}
The incompleteness factor $\BeCom$, 
or mis-alignment between $\fClass$ and its image $\T\fClass$, 
is the scalar quantity defined as
	\begin{align}
		\sup_{\f\in\fClass} \inf_{\g \in \fClass}
	\frac{\norm{\g - \T\f}{\Dist}}{\norm{\f - \T\f}{\Dist}}
	= \BeCom.
	\end{align}
\end{definition}
In other words, $\BeCom$ represents the maximum \emph{fraction} 
of the Bellman error $\norm{\f-\T\f}{\Dist}$ that is not captured by $\fClass$.
When Bellman completeness holds, 
$\inf_{\g\in\fClass}\norm{\g - \T\f}{\Dist} = 0$ for all $\f\in\fClass$,
and thus $\BeCom = 0$.
In the worst case, $\g$ in the numerator in \cref{def:BeCom}
can at least be chosen equal to $\f$, 
in which case we have $\BeCom = 1$.
More generally, $\BeCom$ is a number between zero and one.
The closer it is to zero, the more Bellman complete the MDP is,
in the sense that completeness gets violated more slowly
when moving away from $\fstar$. 
See \cref{fig:Quadratic} for a visual definition of $\BeCom$.
It can be shown that \cref{def:BeCom} leads to the desired behavior in display in \cref{eqn:BoundedIncrease},
since $\Incompleteness(\radius) / \radius$ can be written as
\begin{align*}
	= \sup_{\f\in\fClass(\radius)} \inf_{\g \in \fClass}
	\frac{\norm{\g - \T\f}{\Dist}}{\radius} 
	\leq 
	\sup_{\f\in\fClass(\radius)} \inf_{\g \in \fClass}
	\frac{\norm{\g - \T\f}{\Dist}}{\norm{\f - \T\f}{\Dist}}
	\leq \BeCom.
\end{align*}

How is \cref{def:BeCom} useful for prediction?
Intuitively, the numerator 
$\inf_{\g \in \fClass} \norm{\g - \T\f}{\Dist}$ in \cref{def:BeCom} 
represents some form of approximation error for the backup $\T\f$;
the division by the denominator scales such approximation error 
with respect to the mean-squared Bellman error,
which is the quantity that we wish to reduce.
When the latter is reduced, the approximation error is also reduced,
and the Bellman backup is more faithfully represented. 
In other words, the approximation error must vanish as we approach $\fstar$.

Another possible connection is with the double-sampling issue \cite{baird1995residual}.
Although the mean-squared Bellman error cannot be accurately estimated
without Bellman completeness 
(see e.g. \cite{duan2021risk} for a recent lower bound), 
$\BeCom$-incompleteness ensures that we can estimate it 
with a certain accuracy relative to its magnitude,
and in particular, more accurately for the important functions that are closer to $\fstar$.

\section{Error Bounds on Bellman-Incomplete MDPs}
In this section we present our main results, 
which are off-policy error bounds on the prediction error $\abs{\PE(\fhat)}$ 
for the minimizer $\fhat$ of the empirical loss $\MinimaxEmp$.
These error bounds apply to the limit point for fitted Q 
when it exists \cite{chen2019information}.

\paragraph{Concentrability}
It is useful to introduce the following concentrability 
coefficient \cite{chen2019information,xie2021bellman},
which represents the increase in the mean-squared Bellman error when
moving from the data-generating distribution $\Dist$ 
to that induced by the target policy $\policy$
\begin{align*}
	\Concentrability 
	= 
	\sup_{\f \in \fClass} 
	\frac{\norm{\f - \T\f}{\policy}^2}{\norm{\f - \T\f}{\Dist}^2}.
\end{align*}
As the proof shall clarify, the minimax procedure indirectly attempts to minimize  
the mean-squared Bellman error over $\Dist$ 
(even though it cannot estimate it properly), 
while the prediction error is related to that over $\dof{\policy}$.
Therefore, the concentrability coefficient\footnote{Some weaker upper bounds, 
which have the advantage of being independent of $\fClass$,
are the following:
$$\Concentrability \leq 
	\E_{\psa \sim \Dist}
	\Big[\frac{\dpi{\policy}\psa}{\Dist\psa}\Big]^2
	\leq
	\sup_{\psa}
	\frac{\dpi{\policy}\psa}{\Dist\psa}. $$} translates how minimizing the mean-squared
Bellman error over $\Dist$ affects that over $\dpi{\policy}$, and hence the prediction error.
The higher the value of $\Concentrability$, the less effective the minimax algorithm is,
because the value of the mean-squared Bellman error over $\Dist$ 
is less representative of the prediction error.

\subsection{Error bounds with finite classes}
\label{sec:EBFiniteClasses}
For simplicity, let us present the main findings first when 
the cardinality of $\fClass$ is finite.
\begin{theorem}[Error Bound with Finite Classes]
\label{thm:MinimaxErrorBoundFC}
With probability at least $1-\FailureProbability$, 
the prediction error of the minimizer $\fhat$ satisfies the bound
\begin{align}
	\label{eqn:MinimaxErrorBoundFC}
	\abs{\PE(\fhat)}
	& \leq 
	\frac{1}{1-\discount}\frac{1}{1-\BeCom} 
	\sqrt{\frac{\Concentrability \ln(\card{\fClass}/\FailureProbability)}{\nSamples}}.
\end{align}
\end{theorem}
The proof is in \cref{sec:MainAnalysis}.
The bound above exhibits a typical dependence on several factors: 
the log failure probability $\ln(1/\FailureProbability)$,
the square-root of the number of samples $\nSamples$,
the effective horizon $\frac{1}{1-\discount}$, 
the metric entropy $\ln(\card{\fClass})$
and the concentrability factor $\Concentrability$.
However, the key novelty is the presence of the pre-factor $\frac{1}{1-\BeCom}$
that measures the lack of Bellman completeness, and the absence of the inherent Bellman error.
Practically speaking, the form of the equation suggests that realizability is sufficient whenever 1) $\BeCom < 1$, 
and 2) the TD method converges.
When $\BeCom =1$, off-policy learning is unviable without additional `domain knowledge' 
because the projected Bellman equations---which TD methods aim to solve---may have multiple solutions.

Compared to the state of the art 
\cite{chen2019information,jin2021bellman,xie2021bellman,duan2021risk}
analyses of the minimax algorithm, the use of the local inherent Bellman errors 
has transformed the approximation error term $\GlobalIBE$
into the pre-factor $\frac{1}{1-\BeCom}$ that multiplies the rate of convergence.
In other words, \cref{eqn:MinimaxErrorBoundFC} establishes that
the lack of Bellman completeness does not generate an approximation error---the
inherent Bellman error---but instead it affects the rate of convergence.

The factor $\frac{1}{1-\BeCom}$ could also be interpreted 
	as the cost, in terms of sample complexity,
	of moving from the double-sampling regime\footnote{
	We say that double samples are available when the available dataset 
	contains two independent transitions for each tuple. 
	More precisely, it contains tuples
	$(\state,\action,\reward,\successorstate, \successorstate_+)$
	such that $\successorstate_+\sim \Transition{}\psa$ 
	and $\successorstate \sim \Transition{}\psa $ are independent successor states,
	a condition hardly met outside of simulated domains or deterministic MDPs.}
 	to the single-sampling regime
	in off-policy reinforcement learning; the work of \cite{duan2021risk} 
	can be used to compare our sample complexity 
	with that of methods based on Bellman residual minimization
	in the double-sampling regime.

It is instructive to examine in more details the three key components  
that determine the sample complexity.
\begin{itemize}[leftmargin=*]
	\item The \textbf{metric entropy}, represented by $\ln(\card{\fClass})$, 
	 arises already in supervised learning \cite{wainwright2019high}.
	\item The \textbf{distribution shift}, represented by the concentrability coefficient $\Concentrability$, arises (as a simplified expression that does not depend on the Bellman operator) if distribution shift is present
	 in supervised learning.
	 \item The \textbf{incompleteness factor}, represented by $\frac{1}{1-\BeCom}$, measures the adequacy of the chosen function class with respect to the Bellman operator $\BellmanOperator$; this is the key factor that 
	 distinguishes the reinforcement learning setting from single-step processes, because it \emph{involves the Bellman operator}.
	 Notice that the notion of $\BeCom$-incompleteness is not an assumption: 
	 the value for $\BeCom$ can always be computed, and its knowledge is not
	 required by the algorithm. 
	 Much like the concentrability coefficient measures the degradation in performance as the target policy $\policy$ visits different state-action pairs than the dataset distribution $\DatasetDist$, 
	 the incompleteness factor $\BeCom$ represents the loss of efficiency as the chosen function class becomes more and more mis-aligned with the Bellman backups.
\end{itemize}

Finally, it is worth to highlight the following fact \cite{chen2019information}:
if fitted Q converges, its limit point must inherit the bound of \cref{thm:MinimaxErrorBoundFC},
and so our completeness-free result applies to the solution found by fitted Q.

\cref{thm:MinimaxErrorBoundFC} already contains the key innovation of this paper.
However, the result only applies to finite classes, which are statistically simple 
but also unstructured: they are non-convex and non-differentiable and hence the above result
cannot be applied to gradient-based methods such as TD.
We deal with more expressive models in \cref{sec:EBGeneral},
and make additional considerations in \cref{sec:OPC}.

\newcommand{\plotsimple}[3]
{
\begin{tikzpicture}[hinge/.style = {fill=white, draw=black}]
\def\normalization{(1+cos(#1))}
\def\scale{2.7}
\def\Mult{1.2}
\def\FclassOffset{0.5}
  \coordinate (f) at (0,0);
  \coordinate (fangle) at (0.5,0);
  \coordinate (Tf)   at ({\scale*cos(#1)/\normalization},{\scale*sin(#1)/\normalization});
  \coordinate (gf)   at ({\scale*cos(#1)/\normalization},0);
  \coordinate (fend)  at ({\Mult*\scale*cos(#1)/\normalization + \FclassOffset},0);
  \draw[->, thick, black] (f)node[left]{$\f$}--(Tf) node[above]{$\T\f$};
  \draw[-, densely dashed ] (Tf)--(gf)node[below ]{#3};
  \draw[->, thick, brown ] (f)--(fend)node[right]{$\fClass$};
  \draw[hinge] (f)circle(2pt);
  \draw[fill=black] (f)circle(1pt);
  \draw[black, fill = black] (Tf)circle(1pt);
  \draw[black, fill = black] (gf)circle(1pt);
  #2
\end{tikzpicture}
}

\begin{figure*}
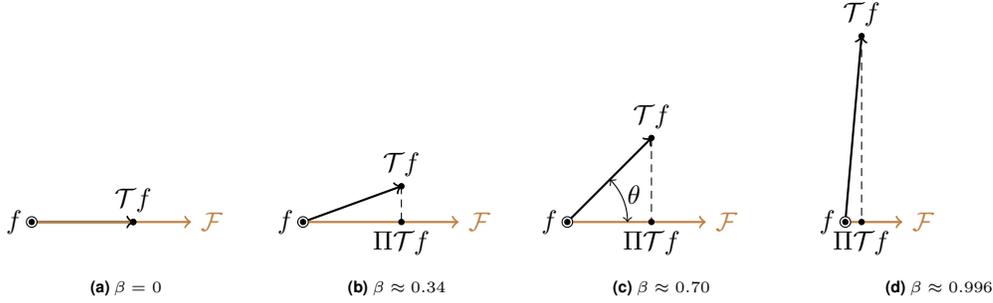

    \centering
    \begin{subfigure}[b]{0.2\textwidth}
    	\plotsimple{0}{}{\vphantom{$\Projector{\T\f}$}}
        \caption{$\BeCom = 0$}
        \label{fig:0}
    \end{subfigure}%
    ~ 
    \begin{subfigure}[b]{0.2\textwidth}
    	\plotsimple{20}{}{$\Projector{\T\f}$}
        \caption{$\BeCom \approx 0.34$}
        \label{fig:20}
    \end{subfigure}
        \begin{subfigure}[b]{0.2\textwidth}
    	\plotsimple{45}
    	{  
    	\pic [draw, <->, angle radius=8mm, angle eccentricity=1.2, "$\theta$"] 
    	{angle = gf--f--Tf}; 
    	}
    	{$\Projector{\T\f}$}
        \caption{$\BeCom \approx 0.70$}
        \label{fig:45}
    \end{subfigure}
    ~ 
    \begin{subfigure}[b]{0.2\textwidth}
    	\plotsimple{85}{}{$\Projector{\T\f}$}
        \caption{$\BeCom \approx 0.996$}
        \label{fig:85}
    \end{subfigure}
    \caption{Local alignments between the Bellman backup $\T\f$ and the class $\fClass$
    for various values of $\BeCom$.
    The setting in \cref{fig:0} is traditionally called `Bellman complete'.
    In this simple example $\BeCom = \sin\theta$.}
    \label{fig:BeCom}
\end{figure*}

\subsection{Comparison with existing guarantees}
In reinforcement learning analyses for model free algorithms, an approximation error term is present even if the problem is realizable,
i.e., even if the action value function $\fstar$ of the target policy is contained in $\fClass$.
Precisely, the approximation error term is the inherent Bellman error of the function class $\fClass$.
A typical bound\footnote{Notice that these papers study the case where $\BellmanOperator$ is the Bellman optimality operator, which leads to slightly different expressions.} \cite{munos2008finite,chen2019information} for the minimax variant reads
\begin{align}
	\label{eqn:OffPolicyPredictionIBE}
	\abs{ \Vpi{\policy} - \Vhat^\policy}  \lessapprox 
	\underbrace{\frac{1}{1-\discount} \sqrt{\frac{\Concentrability 
	\ln(\card{\fClass}/\FailureProbability)}{\nSamples}}}_{\text{stat error}}
	+ 
	\underbrace{ \frac{\sqrt{\Concentrability}}{1-\discount}  \vphantom{\sqrt{\frac{\Concentrability \ln\card{\fClass}}{\nSamples}}} 
	\GlobalIBE.}_{\text{approx error}}
\end{align}
According to \cref{eqn:OffPolicyPredictionIBE},
the prediction error can be reduced only up to an 
error floor represented by the inherent Bellman error $\GlobalIBE$
of the function class $\fClass$.

\begin{SCfigure}
		\begin{tikzpicture}[hinge/.style = {fill=white, draw=black}]
		\def\localvarlimit{88}
		\def\Tfangle{60}
		\def\sc{0.14\textwidth}
		\def\Mult{1.2}
		\def\MultTf{1.1}
		\def\mycolor{red!80!black}
		\def\FclassOffset{5}
		  \coordinate (o) at (0,0);
		  \coordinate (upperlimit)   at ({\sc*cos(90)},{\sc*sin(90)});
		  \coordinate (Tflimit)   at ({\sc*cos(\localvarlimit)},{\sc*sin(\localvarlimit)});
		  \coordinate (Tf)   at ({\MultTf*\sc*cos(\Tfangle)},{\MultTf*\sc*sin(\Tfangle)});
		  \coordinate (Tfend)  at ({\sc*cos(0)},0);
		  \draw[gray] (Tfend) arc (00:90:\sc);
		  \draw[-, thick, \mycolor] (o)node[left]{}--(upperlimit) node[above left]{$\BeCom = 1$};
		  \draw[-, thick, black] (o)node[left]{}--(Tflimit) node[above]{};
		  \draw[->, thick, black ] (o)--(Tfend)node[right]{$\BeCom = 0$};
		  \draw[->, thick, black ] (o)--(Tf)node[right]{$\BeCom$};
		  \draw[hinge] (o)circle(2pt);
		  \draw[fill=black] (o)circle(1pt);
		  \draw[\mycolor, fill = \mycolor] (upperlimit)circle(1pt);
		  \draw[\mycolor, fill = \mycolor] (Tflimit)circle(1pt);
		  \draw[thick,\mycolor,pattern=north east lines,pattern color=\mycolor] 
		  (o) -- (Tflimit)--(upperlimit);
		\end{tikzpicture}
  \caption{Off-policy reinforcement learning remains viable  
  for values of $\BeCom$ in the range $[0,1)$,
  while prior analyses expected an unavoidable inherent Bellman error to arise.
  The red shaded area, which corresponds to $\BeCom \rightarrow 1$,
  represents problems where the sample complexity becomes unmanageably large,
  a condition in force in the lower bounds.}
  \label{fig:Learnability}
\end{SCfigure}
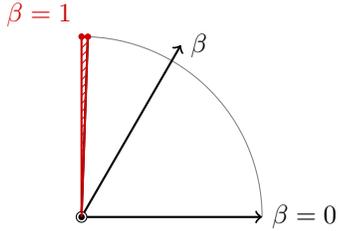

\cref{fig:0,fig:20,fig:45,fig:85} display some Bellman errors 
to help appreciate the results of this paper
and the informal definition of $\BeCom$.
When Bellman completeness holds such as in \cref{fig:0}, 
the class $\fClass$ fully captures the Bellman backup and thus $\BeCom = 0$ 
(no component of the Bellman error is left un-captured).
In this case, the existing bound in \cref{eqn:OffPolicyPredictionIBE}
and the new one in \cref{thm:MinimaxErrorBoundFC}
both reduce to $\abs{\Vpi{\policy} - \Vhat^\policy}  \lessapprox \frac{1}{1-\discount}
\sqrt{\frac{\Concentrability \ln(\card{\fClass}/\FailureProbability)}{\nSamples}} $.

The difference between the new analysis and the existing ones becomes stark when 
completeness is violated.
For example, in \cref{fig:45}, the Bellman backup $\T\f$ is mis-aligned with respect to $\fClass$,
and the residual in \cref{eqn:residual} can be quite large if the Bellman error $\f - \T\f$
is also large. 
For the specific example in \cref{fig:45}, 
the residual in \cref{eqn:residual}
is roughly a fraction $\BeCom \approx 0.7$ of the full Bellman error,
i.e., $\inf_{\g \in \fClass}\norm{\g - \T\f}{\Dist} \approx \BeCom \norm{\f - \T\f}{\Dist}$.
If the Bellman error happens to be large, say $\norm{\f - \T\f}{\Dist} \approx 1$,
then the residual $\inf_{\g \in \fClass}\norm{\g - \T\f}{\Dist}$ will also be large.
It follows that the inherent Bellman error will be large as well,
and so will the prediction error when estimated via \cref{eqn:OffPolicyPredictionIBE}:
\begin{align}
\label{eqn:IntroExIBE}
	\abs{ \Vpi{\policy} - \Vhat^\policy}  \gtrapprox \GlobalIBE \approx 1.
\end{align}  
In other words, the bound \ref{eqn:OffPolicyPredictionIBE} becomes vacuous.
However, if the situation depicted in \cref{fig:45} is representative 
of the mutual alignment between $\T\f$ and $\fClass$ across various $\f \in \fClass$, 
then in lieu of a large approximation error, 
our analysis predicts only a slowdown 
of a factor of $\frac{1}{1-\BeCom} \approx 3$ compared to the Bellman complete case:
\begin{align}
	\label{eqn:IntroExAlCom}
	\abs{\Vpi{\policy} - \Vhat^\policy}  \lessapprox 
	\underbrace{3}_{ \frac{1}{1-\BeCom} } \times 
	\frac{1}{1-\discount}\sqrt{\frac{\Concentrability \ln (\card{\fClass}/\FailureProbability) }{\nSamples}}.
\end{align}
For such problems, the bound in display in \cref{eqn:IntroExAlCom}  
is a major improvement compared to the one in \cref{eqn:IntroExIBE}. 
While the analyses that lead to \cref{eqn:IntroExIBE} suggest that accurate predictions 
are out of reach due to large inherent Bellman errors,
the refined one of this paper expects a 
minor slowdown in the rate of convergence compared to the Bellman complete case.

It is only when the Bellman backup becomes almost orthogonal to $\fClass$
that $\BeCom$ approaches one and prediction becomes very challenging;
such is the situation depicted in \cref{fig:85} and in force in some recent lower bounds
(e.g., \cite{foster2021offline}).
See \cref{fig:Learnability} for a graphical summary.
More precisely, the condition $\BeCom = 1$ corresponds to the existence of multiple 
projected fixed points.
Any method based on finding projected fixed points to the Bellman equations necessarily 
fails to converge to the correct predictor on such problems,
because the correct predictor is only one of the many possible solutions to the projected Bellman equations.

When $\BeCom$ is close to one, 
the classical bound in
\cref{eqn:OffPolicyPredictionIBE}
can be tighter than 
the new bound in \cref{eqn:MinimaxErrorBoundFC}.
Of course, one can always select the tighter of the two.
Likewise, it is possible to leverage the more general notion of 
local inherent Bellman error instead of that of $\BeCom$-incompleteness
and achieve tighter error guarantees than the ones that we present,
but doing so would have only been possible at the expense of the clarity of exposition.
Instead, the key contribution of this work
is to interpret the inherent Bellman error 
no longer as an unavoidable approximation error
that must be zero for the approximation error to be zero,
but as a quantity that naturally decreases when more samples are added.
More precisely, if $\BeCom < 1$, as the number of samples
$\nSamples$ increases, the bound in 
\cref{eqn:MinimaxErrorBoundFC}
eventually becomes tighter than that in
\cref{eqn:OffPolicyPredictionIBE},
establishing convergence to the optimal predictor
even when the inherent Bellman error is non-zero.
See also \cref{sec:FurtherComments}.

\subsection{Further comparison with existing literature}
\label{sec:Literature}
One work close to ours is \cite{xie2020batch},
which operates with stronger concentrability requirements. 
Another one is the non-linear Bubnov-Galerkin method \cite{zanette2022Bellman},
for which we may expect similar considerations to apply;
however, the violation of completeness is not quantified 
in an interpretable way in that work.

Our result is due to a refined analysis, as well as to an appropriate definition, 
and not to a new algorithm.
The minimax formulation has been analyzed multiple times,
\cite{antos2008learning,chen2019information,xie2021bellman,jin2021bellman,duan2021risk,xie2022role}
but to our knowledge all analyses use the inherent Bellman errors.
Although our minimax formulation is for policy evaluation,
as the proof will clarify, the same argument applies to policy optimization
(i.e., when $\T$ is the Bellman optimality operator). 
Finally, our work removes the binary distinction between Bellman completeness 
and the lower bound of \cite{foster2021offline}.

\paragraph{Additional literature}
The off-policy prediction task has been widely studied.
Earlier methods where based on temporal difference (TD) \cite{sutton1988learning};
they include $Q$-learning \cite{watkins1992q} 
and fitted Q iteration \cite{ernst2005tree,munos2008finite}.
These TD methods are key to the recent successes of RL
\cite{tesauro1995temporal,mnih2013playing,mnih2015human,mnih2016asynchronous,fujimoto2018addressing}.

A more robust TD variant which is however harder to optimize numerically
is the minimax formulation that we investigate here \cite{antos2008learning};
its relation with TD methods has been investigated by \cite{chen2019information}.
The minimax formulations has been adopted for provably efficient exploration \cite{jin2021bellman}
and offline robust optimization \cite{xie2021bellman}.
More recently, the minimax formulation has been used as a proxy to analyze theoretically 
an empirical algorithm based on TD \cite{cheng2022adversarially}.
An analysis based on local Rademacher averages is given in \cite{duan2021risk}.
All these analyses require Bellman completeness, 
or otherwise the inherent Bellman error must be suffered.

Many other algorithms for the off-policy prediction problems have been proposed.
These include importance sampling methods 
\cite{precup2000eligibility,thomas2016data,jiang2016doubly,liu2018breaking,farajtabar2018more},
which do not require completeness but can only tolerate small distribution shifts.

More recent literature has proposed weight-learning methods
which rely on the knowledge of certain weights, 
typically the marginalized importance ratios between the distribution that collected the data
and the target policy
\cite{liu2018breaking,xie2020Q,zhan2022offline,nachum2019algaedice,xie2019towards,zhang2020gendice,zhang2020gradientdice,yang2020off,kallus2019efficiently,jiang2020minimax,uehara2020minimax,zanette2022Bellman,rashidinejad2022optimal}.
While these algorithms can avoid Bellman completeness,
they rely on additional assumptions, such as realizability of the weight class,
and more generally they leverage additional domain knowledge 
which is implicit in the choice of the weight class.
For example, \cite{uehara2021finite} 
makes completeness assumptions about the weight class, 
and \cite{zhan2022offline} assume realizability for both the weight and value class.
An additional high-level viewpoint is presented in \cref{app:Notation_Additional}.

Two notable exceptions to completeness are \cite{xie2020batch,zanette2022Bellman};
however \cite{xie2020batch} make very strong assumptions on the concentrability factor,
while the violation of the completeness condition is not quantified in \cite{zanette2022Bellman}.
The violation of completeness is also examined algebraically and algorithmically 
for the linear setting by \cite{perdomo2022sharp}.
For off-policy learning with pessimism and linear methods, completeness was removed 
via a Bubnov-Galerkin approach in \cite{zanette2022Bellman}
while still ensuring computational tractability; in contrast, 
here we focus on more general non-linear predictors.

Fundamental limits were investigated in 
\cite{zanette2020exponential,wang2020statistical,foster2021offline}.
Collectively they show that hard-to-learn structures can arise in absence of Bellman completeness,
or with large distribution shift.
Our paper describes the intermediate situation between 
these lower bounds and the Bellman complete setting.
Related papers include
\cite{duan2020minimax,duan2021risk,tang2019doubly,nachum2020reinforcement,uehara2021finite,chen2022well,chang2022learning}.

Other papers have implicitly examined settings that are intermediate 
between realizability and completeness, such as \cite{wei2022model,ye2022corruption}.
In their setting, if the corruption continues through time then
the regret scales linearly.
Rather, our setting is corruption free, and we can indeed converge to the optimal solution when $\BeCom < 1$.

\section{Conclusion}
In this work we have re-analyzed the statistical complexity of off-policy reinforcement learning
on Bellman-incomplete MDPs using temporal-difference-style algorithms. 
The work establishes that there exists a full spectrum between Bellman completeness
and the existing lower bounds where off-policy reinforcement learning
remains statistically viable, even without additional domain knowledge,
such as weights or test classes, and with no approximation error.
The key advancement is due to a localization argument,
which removes the approximation error associated to the lack of Bellman completeness.

Even though we presented our findings for the policy evaluation problem,
the optimization setting is immediately covered 
by replacing the Bellman evaluation operator with its optimization counterpart;
since our main analysis only relies on the boundedness of the Bellman evaluation operator,
this is a straightforward operation.
We also expect these insights to extend directly 
to the setting of exploration and of pessimistic policy learning.
More generally, we believe that a local analysis 
can be a useful tool to analyze new algorithms or existing ones in other settings as well.
It can help carefully assess how the violation of a certain assumption
affects the performance of an algorithm,
so as to relax some structural assumptions in a way that does not introduce an approximation error.

Finally, although our paper exhibits an algorithm to 
find high-quality solutions in absence of Bellman completeness,
there is no guarantee that such points can be found in a computationally efficient way.
For example, TD methods do not always converge, although when they do, they inherit such bounds.
That raises an interesting question, one that concerns possible statistical-computational
trade-offs to be made in reinforcement learning.

\section*{Acknowledgments}
The author is grateful to the reviewers for their helpful comments which helped improve the quality of the paper.
The author is supported by a fellowship from the Foundation of Data Science Institute (FODSI).

\label{submission}
\bibliography{rl.bib}

\begin{thebibliography}{62}
\providecommand{\natexlab}[1]{#1}
\providecommand{\url}[1]{\texttt{#1}}
\expandafter\ifx\csname urlstyle\endcsname\relax
  \providecommand{\doi}[1]{doi: #1}\else
  \providecommand{\doi}{doi: \begingroup \urlstyle{rm}\Url}\fi

\bibitem[Antos et~al.(2008)Antos, Szepesv{\'a}ri, and Munos]{antos2008learning}
Antos, A., Szepesv{\'a}ri, C., and Munos, R.
\newblock Learning near-optimal policies with {B}ellman-residual minimization
  based fitted policy iteration and a single sample path.
\newblock \emph{Machine Learning}, 71\penalty0 (1):\penalty0 89--129, 2008.

\bibitem[Baird(1995)]{baird1995residual}
Baird, L.
\newblock Residual algorithms: Reinforcement learning with function
  approximation.
\newblock In \emph{International Conference on Machine Learning (ICML)}. 1995.

\bibitem[Bartlett et~al.(2005)Bartlett, Bousquet, and
  Mendelson]{bartlett2005local}
Bartlett, P.~L., Bousquet, O., and Mendelson, S.
\newblock Local rademacher complexities.
\newblock \emph{The Annals of Statistics}, 33\penalty0 (4):\penalty0
  1497--1537, 2005.

\bibitem[Bertsekas(1995{\natexlab{a}})]{Bertsekas_dyn2}
Bertsekas, D.
\newblock \emph{Dynamic programming and stochastic control}, volume~2.
\newblock Athena Scientific, Belmont, MA, 1995{\natexlab{a}}.

\bibitem[Bertsekas(1995{\natexlab{b}})]{Bertsekas_dyn1}
Bertsekas, D.~P.
\newblock \emph{Dynamic programming and stochastic control}, volume~1.
\newblock Athena Scientific, Belmont, MA, 1995{\natexlab{b}}.

\bibitem[Bertsekas \& Tsitsiklis(1996)Bertsekas and
  Tsitsiklis]{bertsekas1996neuro}
Bertsekas, D.~P. and Tsitsiklis, J.~N.
\newblock \emph{Neuro-dynamic programming}.
\newblock Athena Scientific, 1996.

\bibitem[Chang et~al.(2022)Chang, Wang, Kallus, and Sun]{chang2022learning}
Chang, J., Wang, K., Kallus, N., and Sun, W.
\newblock Learning bellman complete representations for offline policy
  evaluation.
\newblock In \emph{International Conference on Machine Learning}, pp.\
  2938--2971. PMLR, 2022.

\bibitem[Chen \& Jiang(2019)Chen and Jiang]{chen2019information}
Chen, J. and Jiang, N.
\newblock Information-theoretic considerations in batch reinforcement learning.
\newblock In \emph{International Conference on Machine Learning}, pp.\
  1042--1051, 2019.

\bibitem[Chen \& Qi(2022)Chen and Qi]{chen2022well}
Chen, X. and Qi, Z.
\newblock On well-posedness and minimax optimal rates of nonparametric
  q-function estimation in off-policy evaluation.
\newblock \emph{arXiv preprint arXiv:2201.06169}, 2022.

\bibitem[Cheng et~al.(2022)Cheng, Xie, Jiang, and
  Agarwal]{cheng2022adversarially}
Cheng, C.-A., Xie, T., Jiang, N., and Agarwal, A.
\newblock Adversarially trained actor critic for offline reinforcement
  learning.
\newblock \emph{arXiv preprint arXiv:2202.02446}, 2022.

\bibitem[Duan \& Wang(2020)Duan and Wang]{duan2020minimax}
Duan, Y. and Wang, M.
\newblock Minimax-optimal off-policy evaluation with linear function
  approximation.
\newblock \emph{arXiv preprint arXiv:2002.09516}, 2020.

\bibitem[Duan et~al.(2021)Duan, Jin, and Li]{duan2021risk}
Duan, Y., Jin, C., and Li, Z.
\newblock Risk bounds and rademacher complexity in batch reinforcement
  learning.
\newblock \emph{arXiv preprint arXiv:2103.13883}, 2021.

\bibitem[Ernst et~al.(2005)Ernst, Geurts, and Wehenkel]{ernst2005tree}
Ernst, D., Geurts, P., and Wehenkel, L.
\newblock Tree-based batch mode reinforcement learning.
\newblock \emph{Journal of Machine Learning Research}, 6, 2005.

\bibitem[Fan et~al.(2020)Fan, Wang, Xie, and Yang]{fan2020theoretical}
Fan, J., Wang, Z., Xie, Y., and Yang, Z.
\newblock A theoretical analysis of deep q-learning.
\newblock In \emph{Learning for Dynamics and Control}, pp.\  486--489. PMLR,
  2020.

\bibitem[Farajtabar et~al.(2018)Farajtabar, Chow, and
  Ghavamzadeh]{farajtabar2018more}
Farajtabar, M., Chow, Y., and Ghavamzadeh, M.
\newblock More robust doubly robust off-policy evaluation.
\newblock In \emph{International Conference on Machine Learning}, pp.\
  1447--1456. PMLR, 2018.

\bibitem[Foster et~al.(2021)Foster, Krishnamurthy, Simchi-Levi, and
  Xu]{foster2021offline}
Foster, D.~J., Krishnamurthy, A., Simchi-Levi, D., and Xu, Y.
\newblock Offline reinforcement learning: Fundamental barriers for value
  function approximation, 2021.

\bibitem[Fujimoto et~al.(2018)Fujimoto, Hoof, and
  Meger]{fujimoto2018addressing}
Fujimoto, S., Hoof, H., and Meger, D.
\newblock Addressing function approximation error in actor-critic methods.
\newblock In \emph{International conference on machine learning}, pp.\
  1587--1596. PMLR, 2018.

\bibitem[Jiang \& Huang(2020)Jiang and Huang]{jiang2020minimax}
Jiang, N. and Huang, J.
\newblock Minimax value interval for off-policy evaluation and policy
  optimization.
\newblock \emph{arXiv preprint arXiv:2002.02081}, 2020.

\bibitem[Jiang \& Li(2016)Jiang and Li]{jiang2016doubly}
Jiang, N. and Li, L.
\newblock Doubly robust off-policy value evaluation for reinforcement learning.
\newblock In \emph{International Conference on Machine Learning}, pp.\
  652--661. PMLR, 2016.

\bibitem[Jin et~al.(2021)Jin, Liu, and Miryoosefi]{jin2021bellman}
Jin, C., Liu, Q., and Miryoosefi, S.
\newblock {B}ellman eluder dimension: New rich classes of rl problems, and
  sample-efficient algorithms.
\newblock \emph{arXiv preprint arXiv:2102.00815}, 2021.

\bibitem[Kakade et~al.(2003)]{kakade2003sample}
Kakade, S.~M. et~al.
\newblock \emph{On the sample complexity of reinforcement learning}.
\newblock PhD thesis, University of London London, England, 2003.

\bibitem[Kallus \& Uehara(2019)Kallus and Uehara]{kallus2019efficiently}
Kallus, N. and Uehara, M.
\newblock Efficiently breaking the curse of horizon in off-policy evaluation
  with double reinforcement learning.
\newblock \emph{arXiv preprint arXiv:1909.05850}, 2019.

\bibitem[Liu et~al.(2018)Liu, Li, Tang, and Zhou]{liu2018breaking}
Liu, Q., Li, L., Tang, Z., and Zhou, D.
\newblock Breaking the curse of horizon: Infinite-horizon off-policy
  estimation.
\newblock In \emph{Advances in Neural Information Processing Systems}, pp.\
  5356--5366, 2018.

\bibitem[Mnih et~al.(2013)Mnih, Kavukcuoglu, Silver, Graves, Antonoglou,
  Wierstra, and Riedmiller]{mnih2013playing}
Mnih, V., Kavukcuoglu, K., Silver, D., Graves, A., Antonoglou, I., Wierstra,
  D., and Riedmiller, M.
\newblock Playing atari with deep reinforcement learning.
\newblock \emph{arXiv preprint arXiv:1312.5602}, 2013.

\bibitem[Mnih et~al.(2015)Mnih, Kavukcuoglu, Silver, Rusu, Veness, Bellemare,
  Graves, Riedmiller, Fidjeland, Ostrovski, et~al.]{mnih2015human}
Mnih, V., Kavukcuoglu, K., Silver, D., Rusu, A.~A., Veness, J., Bellemare,
  M.~G., Graves, A., Riedmiller, M., Fidjeland, A.~K., Ostrovski, G., et~al.
\newblock Human-level control through deep reinforcement learning.
\newblock \emph{nature}, 518\penalty0 (7540):\penalty0 529--533, 2015.

\bibitem[Mnih et~al.(2016)Mnih, Badia, Mirza, Graves, Lillicrap, Harley,
  Silver, and Kavukcuoglu]{mnih2016asynchronous}
Mnih, V., Badia, A.~P., Mirza, M., Graves, A., Lillicrap, T., Harley, T.,
  Silver, D., and Kavukcuoglu, K.
\newblock Asynchronous methods for deep reinforcement learning.
\newblock In \emph{International conference on machine learning}, pp.\
  1928--1937. PMLR, 2016.

\bibitem[Munos(2005)]{munos2005error}
Munos, R.
\newblock Error bounds for approximate value iteration.
\newblock In \emph{AAAI Conference on Artificial Intelligence (AAAI)}, 2005.

\bibitem[Munos \& Szepesv{\'a}ri(2008)Munos and
  Szepesv{\'a}ri]{munos2008finite}
Munos, R. and Szepesv{\'a}ri, C.
\newblock Finite-time bounds for fitted value iteration.
\newblock \emph{Journal of Machine Learning Research}, 9\penalty0
  (May):\penalty0 815--857, 2008.

\bibitem[Nachum \& Dai(2020)Nachum and Dai]{nachum2020reinforcement}
Nachum, O. and Dai, B.
\newblock Reinforcement learning via {F}enchel-{R}ockafellar duality.
\newblock \emph{arXiv preprint arXiv:2001.01866}, 2020.

\bibitem[Nachum et~al.(2019)Nachum, Dai, Kostrikov, Chow, Li, and
  Schuurmans]{nachum2019algaedice}
Nachum, O., Dai, B., Kostrikov, I., Chow, Y., Li, L., and Schuurmans, D.
\newblock Algaedice: Policy gradient from arbitrary experience.
\newblock \emph{arXiv preprint arXiv:1912.02074}, 2019.

\bibitem[Perdomo et~al.(2022)Perdomo, Krishnamurthy, Bartlett, and
  Kakade]{perdomo2022sharp}
Perdomo, J.~C., Krishnamurthy, A., Bartlett, P., and Kakade, S.
\newblock A sharp characterization of linear estimators for offline policy
  evaluation.
\newblock \emph{arXiv preprint arXiv:2203.04236}, 2022.

\bibitem[Precup(2000)]{precup2000eligibility}
Precup, D.
\newblock Eligibility traces for off-policy policy evaluation.
\newblock \emph{Computer Science Department Faculty Publication Series}, pp.\
  ~80, 2000.

\bibitem[Puterman(1994)]{puterman1994markov}
Puterman, M.~L.
\newblock \emph{Markov Decision Processes: Discrete Stochastic Dynamic
  Programming}.
\newblock John Wiley \& Sons, Inc., New York, NY, USA, 1994.
\newblock ISBN 0471619779.

\bibitem[Rashidinejad et~al.(2022)Rashidinejad, Zhu, Yang, Russell, and
  Jiao]{rashidinejad2022optimal}
Rashidinejad, P., Zhu, H., Yang, K., Russell, S., and Jiao, J.
\newblock Optimal conservative offline rl with general function approximation
  via augmented lagrangian.
\newblock \emph{arXiv preprint arXiv:2211.00716}, 2022.

\bibitem[Shalev-Shwartz \& Ben-David(2014)Shalev-Shwartz and
  Ben-David]{shalev2014understanding}
Shalev-Shwartz, S. and Ben-David, S.
\newblock \emph{Understanding machine learning: From theory to algorithms}.
\newblock Cambridge university press, 2014.

\bibitem[Sutton(1988)]{sutton1988learning}
Sutton, R.~S.
\newblock Learning to predict by the methods of temporal differences.
\newblock \emph{Machine learning}, 3\penalty0 (1):\penalty0 9--44, 1988.

\bibitem[Sutton \& Barto(2018)Sutton and Barto]{sutton2018reinforcement}
Sutton, R.~S. and Barto, A.~G.
\newblock \emph{Reinforcement learning: An introduction}.
\newblock MIT Press, 2018.

\bibitem[Talagrand(1996)]{talagrand1996new}
Talagrand, M.
\newblock A new look at independence.
\newblock \emph{The Annals of probability}, pp.\  1--34, 1996.

\bibitem[Tang et~al.(2019)Tang, Feng, Li, Zhou, and Liu]{tang2019doubly}
Tang, Z., Feng, Y., Li, L., Zhou, D., and Liu, Q.
\newblock Doubly robust bias reduction in infinite horizon off-policy
  estimation.
\newblock \emph{arXiv preprint arXiv:1910.07186}, 2019.

\bibitem[Tesauro et~al.(1995)]{tesauro1995temporal}
Tesauro, G. et~al.
\newblock Temporal difference learning and td-gammon.
\newblock \emph{Communications of the ACM}, 38\penalty0 (3):\penalty0 58--68,
  1995.

\bibitem[Thomas \& Brunskill(2016)Thomas and Brunskill]{thomas2016data}
Thomas, P. and Brunskill, E.
\newblock Data-efficient off-policy policy evaluation for reinforcement
  learning.
\newblock In \emph{International Conference on Machine Learning}, pp.\
  2139--2148, 2016.

\bibitem[Uehara et~al.(2020)Uehara, Huang, and Jiang]{uehara2020minimax}
Uehara, M., Huang, J., and Jiang, N.
\newblock Minimax weight and q-function learning for off-policy evaluation.
\newblock In \emph{International Conference on Machine Learning}, pp.\
  9659--9668. PMLR, 2020.

\bibitem[Uehara et~al.(2021)Uehara, Imaizumi, Jiang, Kallus, Sun, and
  Xie]{uehara2021finite}
Uehara, M., Imaizumi, M., Jiang, N., Kallus, N., Sun, W., and Xie, T.
\newblock Finite sample analysis of minimax offline reinforcement learning:
  Completeness, fast rates and first-order efficiency.
\newblock \emph{arXiv preprint arXiv:2102.02981}, 2021.

\bibitem[Wainwright(2019)]{wainwright2019high}
Wainwright, M.~J.
\newblock \emph{High-dimensional statistics: A non-asymptotic viewpoint},
  volume~48.
\newblock Cambridge University Press, 2019.

\bibitem[Wang et~al.(2020)Wang, Foster, and Kakade]{wang2020statistical}
Wang, R., Foster, D.~P., and Kakade, S.~M.
\newblock What are the statistical limits of offline rl with linear function
  approximation?
\newblock \emph{arXiv preprint arXiv:2010.11895}, 2020.

\bibitem[Wang et~al.(2021)Wang, Wang, and Kakade]{wang2021exponential}
Wang, Y., Wang, R., and Kakade, S.~M.
\newblock An exponential lower bound for linearly-realizable {MDP}s with
  constant suboptimality gap.
\newblock \emph{arXiv preprint arXiv:2103.12690}, 2021.

\bibitem[Watkins \& Dayan(1992)Watkins and Dayan]{watkins1992q}
Watkins, C.~J. and Dayan, P.
\newblock Q-learning.
\newblock \emph{Machine learning}, 8\penalty0 (3-4):\penalty0 279--292, 1992.

\bibitem[Wei et~al.(2022)Wei, Dann, and Zimmert]{wei2022model}
Wei, C.-Y., Dann, C., and Zimmert, J.
\newblock A model selection approach for corruption robust reinforcement
  learning.
\newblock In \emph{International Conference on Algorithmic Learning Theory},
  pp.\  1043--1096. PMLR, 2022.

\bibitem[Weisz et~al.(2020)Weisz, Amortila, and
  Szepesv{\'a}ri]{weisz2020exponential}
Weisz, G., Amortila, P., and Szepesv{\'a}ri, C.
\newblock Exponential lower bounds for planning in {M}{D}{P}s with
  linearly-realizable optimal action-value functions.
\newblock \emph{arXiv preprint arXiv:2010.01374}, 2020.

\bibitem[Xie \& Jiang(2020{\natexlab{a}})Xie and Jiang]{xie2020Q}
Xie, T. and Jiang, N.
\newblock Q* approximation schemes for batch reinforcement learning: A
  theoretical comparison.
\newblock volume 124 of \emph{Proceedings of Machine Learning Research}, pp.\
  550--559, Virtual, 03--06 Aug 2020{\natexlab{a}}. PMLR.
\newblock URL \url{http://proceedings.mlr.press/v124/xie20a.html}.

\bibitem[Xie \& Jiang(2020{\natexlab{b}})Xie and Jiang]{xie2020batch}
Xie, T. and Jiang, N.
\newblock Batch value-function approximation with only realizability.
\newblock \emph{arXiv preprint arXiv:2008.04990}, 2020{\natexlab{b}}.

\bibitem[Xie et~al.(2019)Xie, Ma, and Wang]{xie2019towards}
Xie, T., Ma, Y., and Wang, Y.-X.
\newblock Towards optimal off-policy evaluation for reinforcement learning with
  marginalized importance sampling.
\newblock In \emph{Advances in Neural Information Processing Systems}, pp.\
  9668--9678, 2019.

\bibitem[Xie et~al.(2021)Xie, Cheng, Jiang, Mineiro, and
  Agarwal]{xie2021bellman}
Xie, T., Cheng, C.-A., Jiang, N., Mineiro, P., and Agarwal, A.
\newblock {B}ellman-consistent pessimism for offline reinforcement learning.
\newblock \emph{arXiv preprint arXiv:2106.06926}, 2021.

\bibitem[Xie et~al.(2022)Xie, Foster, Bai, Jiang, and Kakade]{xie2022role}
Xie, T., Foster, D.~J., Bai, Y., Jiang, N., and Kakade, S.~M.
\newblock The role of coverage in online reinforcement learning.
\newblock \emph{arXiv preprint arXiv:2210.04157}, 2022.

\bibitem[Yang et~al.(2020)Yang, Nachum, Dai, Li, and Schuurmans]{yang2020off}
Yang, M., Nachum, O., Dai, B., Li, L., and Schuurmans, D.
\newblock Off-policy evaluation via the regularized lagrangian.
\newblock \emph{arXiv preprint arXiv:2007.03438}, 2020.

\bibitem[Ye et~al.(2022)Ye, Xiong, Gu, and Zhang]{ye2022corruption}
Ye, C., Xiong, W., Gu, Q., and Zhang, T.
\newblock Corruption-robust algorithms with uncertainty weighting for nonlinear
  contextual bandits and markov decision processes.
\newblock \emph{arXiv preprint arXiv:2212.05949}, 2022.

\bibitem[Zanette(2020)]{zanette2020exponential}
Zanette, A.
\newblock Exponential lower bounds for batch reinforcement learning: Batch rl
  can be exponentially harder than online {R}{L}.
\newblock \emph{arXiv preprint arXiv:2012.08005}, 2020.

\bibitem[Zanette \& Wainwright(2022)Zanette and Wainwright]{zanette2022Bellman}
Zanette, A. and Wainwright, M.~J.
\newblock Bellman residual orthogonalization for offline reinforcement
  learning, 2022.
\newblock URL \url{https://arxiv.org/abs/2203.12786}.

\bibitem[Zanette et~al.(2020)Zanette, Lazaric, Kochenderfer, and
  Brunskill]{zanette2020learning}
Zanette, A., Lazaric, A., Kochenderfer, M., and Brunskill, E.
\newblock Learning near optimal policies with low inherent {B}ellman error.
\newblock In \emph{International Conference on Machine Learning (ICML)}, 2020.

\bibitem[Zhan et~al.(2022)Zhan, Huang, Huang, Jiang, and Lee]{zhan2022offline}
Zhan, W., Huang, B., Huang, A., Jiang, N., and Lee, J.~D.
\newblock Offline reinforcement learning with realizability and single-policy
  concentrability.
\newblock \emph{arXiv preprint arXiv:2202.04634}, 2022.

\bibitem[Zhang et~al.(2020{\natexlab{a}})Zhang, Dai, Li, and
  Schuurmans]{zhang2020gendice}
Zhang, R., Dai, B., Li, L., and Schuurmans, D.
\newblock Gendice: Generalized offline estimation of stationary values.
\newblock \emph{arXiv preprint arXiv:2002.09072}, 2020{\natexlab{a}}.

\bibitem[Zhang et~al.(2020{\natexlab{b}})Zhang, Liu, and
  Whiteson]{zhang2020gradientdice}
Zhang, S., Liu, B., and Whiteson, S.
\newblock Gradientdice: Rethinking generalized offline estimation of stationary
  values.
\newblock In \emph{International Conference on Machine Learning}, pp.\
  11194--11203. PMLR, 2020{\natexlab{b}}.

\end{thebibliography}
\bibliographystyle{icml2023}

\newpage
\appendix
\onecolumn
\addcontentsline{toc}{section}{Appendix} 
\part{Appendix} 
\parttoc
The appendix is organized as follows:
\begin{itemize}
	\item \cref{sec:FurtherComments} presents further comments, particularly related to the weight methods
	\item \cref{app:Notation_Additional} describes additional notation
	\item \cref{sec:AdditionalResult} describes additional results
	\item \cref{sec:MainAnalysis} presents the main proof of the paper
	\item \cref{sec:ProofOfMinimax} presents the main technical sub-component of the paper, which is the rate of the minimax program
	\item \cref{sec:TechRes} presents some technical results needed in the prior sections
\end{itemize}

\section{Further Comments on the Relation between TD and Weight Methods}
\label{sec:FurtherComments}
There is a solid high-level connection between TD and weight methods,
which we discuss in this section.

If one had access to a generative model, the mean-squared Bellman error can be minimized to find a good predictor.
However, without a generative model, it is not possible to directly estimate 
(and thus minimize) the mean square Bellman error when function approximation is implemented. 
In this case, the `standard' approach 
(e.g., temporal difference learning, fitted Q, but also the minimax formulation that we examine here) is to roughly minimize the projected Bellman error.
 To be more precise, the Bellman error is projected onto $\fClass$. 
 Of course, the projection may discard important components of the Bellman error 
 (those orthogonal to $\fClass$), and so there is a loss in sample efficiency, 
 which our work quantifies with the scalar $\BeCom$. 
 When prior art assumed Bellman completeness, they assumed that there are no orthogonal components.

One might wonder whether it makes sense to `project' the Bellman error along different spaces 
(i.e., a space $\mathcal V$ different from $\fClass$). 
This idea roughly leads to the class of weight methods, although they are normally not presented as methods doing projections;
see the paper \cite{zanette2022Bellman} for one such viewpoint.

Which one (TD or weight learning) is better? The answer is problem dependent. 
At a very basic level, if $\fClass$ is well aligned with the Bellman error, 
TD-style methods are superior. If one has specific knowledge of a subspace $\mathcal V$
that better captures the Bellman error, then a weight learning method can be used. 
A special case of this is, for instance, when $\mathcal V$
contains the density ratio of the target policy with respect to the behavioral policy.

While weight learning methods are conceptually appealing, 
it is rare to have such domain knowledge to exploit with a weight learning method, 
and so TD-style methods (broadly those that we analyze here) remain very popular.

\section{Additional Notation}
\label{app:Notation_Additional}
\paragraph{TD and Bellman errors}
For a given $\Q$-function and
policy $\policy$, let us define the \emph{temporal difference error}
(or TD error) associated to the sample $\sars{}$ and
the \emph{Bellman error} at $\psa$
\begin{align}
\label{eqn:TDandBEerr}
(\TD\f)\sars{} \defeq \f\psa - \reward -
  \discount \f(\successorstate,\policy), \qquad
  (\BE\f)\psa \defeq \f\psa - \reward\psa
  - \discount
  \E_{\successorstate\sim\TransitionLaw\psa}\f(\successorstate,\policy).
\end{align}
The TD error is a random variable function of $\sars{}$, while the
Bellman error is its conditional expectation with respect to the
immediate reward and successor state at $\psa$. 

\paragraph{Function class}
We deal with a function class $\fClass$ that contains a set of predictors $\f$
defined over the state and action space.
They are bounded in supremum norm, i.e., $\sup_{\psa} \abs{\f\psa}\leq 1$,
a bound that must apply to $\fstar$ as well since we assume realizability.

Some of our results are presented using a statistical complexity notion 
called Rademacher complexity.
The Rademacher complexity of a function class measures 
the expected worst-case alignment of a predictor $\f \in \fClass$,
evaluated in a $\nSamples$-dimensional space over the random covariates 
$(\SState_\iSample, \AAction_\iSample)\sim\Dist$, with the Rademacher noise $\RadVar_\iSample$,
which takes value $-1$ and $+1$ with equal probability.
It is defined for a function class $\fClass$ as
\begin{align*}
	\Rademacher_\nSamples \big[\fClass\big]
	& =
	\E \sup_{\f \in \fClass } 
	\bigabs{
	\frac{1}{\nSamples} \sum_{\iSample = 1}^{\nSamples}
	\RadVar_\iSample\f(\SState_\iSample,\AAction_\iSample)
	}.
\end{align*}

When presenting our results for general function approximation,
we use a set that contains functions that are at most $\radius \geq 0$ away from the optimal one.
It is defined as
\begin{align}
\label{eqn:LocalizedSet}
	(\fClass - \fstar)(\radius) = \{ \f - \fstar \mid \norm{\f - \fstar}{\Dist} \leq \radius
	, \; \f \in \fClass \}.
\end{align}

\newpage
\section{Additional Results}
\label{sec:AdditionalResult}

\subsection{Proof of \cref{prop:iBehav}}
\label{app:iBehav}
\begin{proof}
Let us focus on the first statement and fix two radii $\radius \leq \radius'$
where $\Incompleteness$ exists.
The supremum $\sup_{\f}$ for $\Incompleteness(\radius)$ is over $\fClass(\radius)$
while for $\Incompleteness(\radius')$ it is over $\fClass(\radius')$;
in both cases, the infimum $\inf_{\g}$ is over the original class $\fClass$.
Since
$\fClass(\radius) \subseteq \fClass(\radius')$,
taken together these observations imply
$$
\Incompleteness(\radius) 
\defeq
\sup_{\f \in \fClass(\radius)}
\inf_{\g \in \fClass}
\norm{\g - \T\f}{\Dist}
\leq
\sup_{\f \in \fClass(\radius')}
\inf_{\g \in \fClass}
\norm{\g - \T\f}{\Dist}
\defeq
\Incompleteness(\radius').
$$
Now, for the second statement: when realizability holds,
the set $\fClass(0) = \{ \f \in \fClass \mid \norm{\f - \T\f}{\Dist} \leq 0 \}$ 
contains at least $\fstar$, and it is hence non-empty. 
The fact that $\Incompleteness(0) = 0$ for a realizable problem then follows from 
$$
\Incompleteness(0) 
\defeq \sup_{\f \in \fClass(0) } \inf_{\g \in \fClass} 
\norm{\g - \T\f}{\Dist}
\leq \sup_{\f \in \fClass(0) }  
\norm{\f - \T\f}{\Dist}
\leq 0.
$$
\end{proof}

\subsection{Off-policy cost coefficient}
\label{sec:OPC}
The error bound in \cref{eqn:MinimaxErrorBoundFC}
can be re-written in a more suggestive way:
\begin{align*}
	\abs{\PE(\fhat)} \leq \frac{1}{1-\discount} 
	\sqrt{\frac{\Concentrability}{1-\BeCom}} \criticalradius,
	\qquad \text{where} \quad
	\criticalradius^2 = \crfc.
\end{align*}
The above regroupment has highlighted the dependence on three key factors.
The first is the \emph{rate of convergence} $\criticalradius$ to zero
of the population-level minimax program $\Minimax$
(as the proof will clarify, we have $\Minimax(\fhat) \lesssim \criticalradius^2$
with high probability).
The other two factors are the concentrability coefficient $\Concentrability$ 
and the lack of Bellman completeness $\frac{1}{1-\BeCom}$.
They relate how minimizing $\Minimax$---represented by $\criticalradius$---affects
the prediction error $\PE(\fhat)$.

A natural question to ask is whether it makes sense to have two factors, rather than a single entity,
to relate the value of the program $\Minimax(\fhat)$ and the prediction error $\PE(\fhat)$.
In fact, it is possible to adopt a more direct approach
and directly measure how minimizing $\Minimax(\f)$ affects the prediction error $\PE(\fhat)$,
and denote the worst-case ratio by $\OPC$:
\begin{align}
	\label{eqn:OPC}
	\OPC \defeq \sup_{\f \in \fClass} \frac{\PE(\f)^2}{\Minimax(\f)}
	\approx
	\frac{\text{quantity of interest}}{\text{quantity being minimized}}.
\end{align}
The off-policy cost coefficient $\OPC$ so defined
always leads to tighter bounds:
it is always smaller than the product between $\Concentrability$
and the incompleteness factor $\frac{1}{1-\BeCom}$
that appears in \cref{thm:MinimaxErrorBoundFC}: 
\begin{align*}
	\OPC 
	\leq 
	\frac{1}{( 1-\discount )^2}
	\frac{\Concentrability}{1-\BeCom}.
\end{align*}
In fact, the proof of \cref{thm:MinimaxErrorBoundFC,thm:MinimaxErrorBoundGeneral}
computes the performance bound of the minimax algorithm using $\OPC$,
only to relax it at the end by using the above display to make the result more interpretable;
one can thus directly replace 
$\frac{1}{(1-\discount)^2} \frac{\Concentrability}{1-\BeCom}$ 
in \cref{eqn:MinimaxErrorBoundFC} and \cref{eqn:MinimaxErrorBoundGeneral} to follow 
with $\OPC$.

Although $\OPC$ is less interpretable in terms of fundamental reinforcement learning quantities, 
its use should be preferred for two reasons.
The first is that it is smaller, i.e., $\OPC$ can be small 
even when $\frac{\Concentrability}{1-\BeCom}$ is large.
The second is that it reflects more truthfully the learning mechanics of the algorithm:
$\OPC$ directly bounds the ratio between 
the quantity of interest---the prediction error $\abs{\PE(\f)}$---and
the one being controlled---the value of the minimax program $\Minimax(\f)$---and 
it is thus the `correct' way to quantify the cost of off-policy learning with the minimax procedure.

\subsection{Error bounds with more general function approximation} 
\label{sec:EBGeneral}
In practice, TD methods are implemented as gradient-based algorithms,
using differentiable approximators that are far more complex
then finite classes and that may operate in a non-parametric regime, such as neural networks.
In such cases, we do not expect a $\sqrt{\nSamples}$ rate of convergence.
In order to provide error bounds that apply to the latter setting,
in this section we express the result using Rademacher averages,
which are standard ways to quantify the capacity of a function class.

As with the localized inherent Bellman error in \cref{sec:Localization},
the relevant sets to determine the statistical complexity---and hence the rate of convergence---are 
subsets of $\fClass$ where we expect the predictor $\fhat$ to be. 
We expect these sets (and their Rademacher complexity) to become smaller as $\nSamples$ increases,
much like the incompleteness function. 

What determines the rate of convergence $\criticalradius$ then is a certain relation
presented in \cref{eqn:RateOfConvergenceGeneral}.
It involves the Rademacher complexity of these localized sets,
which is a standard way to express the rates of convergences with generic function classes
\cite{bartlett2005local,wainwright2019high}.
The conditions in  \cref{eqn:RateOfConvergenceGeneral} 
must admit a solution $\criticalradius$
such that the requirement holds for all $\radius \geq \criticalradius$;
this requirement is met by the bounded classes we consider.
We further assume that there are no measurability issues when stating and proving the following theorem;
in particular we assume that the prerequisites for using Talagrand are met in order to avoid measurability issues.
\begin{theorem}[Error Bounds with General Function Approximation]
\label{thm:MinimaxErrorBoundGeneral}
	With probability at least $1-\FailureProbability$,
	the prediction error of the minimizer $\fhat$ satisfies the bound
	\begin{align}
	\label{eqn:MinimaxErrorBoundGeneral}
		\abs{\PE(\fhat)}
		& \leq 
		\frac{\criticalradius}{1-\discount}\sqrt{\frac{\Concentrability}{1-\BeCom}}
	\end{align} 
	where the rate of convergence $\criticalradius$ is such that all $\radius \geq \criticalradius$
	satisfy the inequalities
	\begin{subequations}
	\label{eqn:RateOfConvergenceGeneral}
	\begin{align}
	\Rademacher_\nSamples 
	\Big\{
	\Cost(\f,\f) - \Cost(\gbest{\f},\f ) 
	\mid 
	\E [\Cost(\f,\f) - \Cost(\gbest{\f},\f)] \leq 2\radius^2
	\Big\}
	& \leq c_1 \radius^2, \\
		\Rademacher_\nSamples
		\Big[(\fClass - \fstar)\big(\factor \radius \big)\Big]
		& \leq c_2 
		\radius^2, \\ 
		(\factor + 1) \frac{\ln(1/( \FailureProbability\radius ))}{\nSamples} 
	& \leq c_3 \radius^2.
	\end{align}
	\end{subequations}
	for three universal constants $c_1, c_2, c_3 > 0$,
	and $\factor = 1$ if $\fClass$ is convex or $\factor = \frac{1}{1-\BeCom}$
	if $\fClass$ is non-convex. 
\end{theorem}
The rate of convergence $\criticalradius$ is that of the minimax procedure,
i.e., we have $\Minimax(\fhat) \lesssim \criticalradius^2$ with high probability.
Let us add that convexity always leads to improved bounds.
The second and third critical inequalities in \cref{eqn:RateOfConvergenceGeneral}
are standard, while the first involves the Bellman operator, and can be relaxed 
only with additional assumptions \cite{duan2021risk}.

In all cases, in order to determine the rate of convergence $\criticalradius$, 
the first step is to compute the local Rademacher averages
in \cref{eqn:RateOfConvergenceGeneral} 
as a function of $\radius$,
and the second step is to solve for $\radius$ the resulting relation,
finding $\criticalradius$.
It is enough to compute an upper bound to the local Rademacher complexity.
Likewise, it is sufficient to identify any value $\criticalradius$
that solves the resulting relation,
but the smaller the $\criticalradius$, the better the rate of convergence
that we can guarantee.
Of course, in order to obtain concrete and interpretable bounds,
one must consider specific function classes,
see the book \cite{wainwright2019high} for several 
parametric as well as non-parametric examples.

Finally, let us mention that 
the bound that we present here uses the coefficient $\BeCom$
which represents the average behavior of $\Incompleteness$,
but intuitively, it is the actual shape of the incompleteness function $\Incompleteness$
around the origin that determines the problem complexity.
It is possible to obtain critical relations involving the incompleteness function,
much like those in \cref{eqn:RateOfConvergenceGeneral}.
However, implementing this observation would have made the analysis less clear
and the final result less interpretable, and so we leave that for future studies. 

\newpage
\section{Main Analysis}
\label{sec:MainAnalysis}
In this section we prove \cref{thm:MinimaxErrorBoundFC,thm:MinimaxErrorBoundGeneral}.

\paragraph{Proof techniques}
Although the minimax formulation has been analyzed previously in a number of works 
(see e.g., \cite{chen2019information} for a relatively recent analysis),
our proof differs from what is available in the literature from the very set-up,
as the concept of local Bellman errors arises quite soon in the proof in \cref{sec:MainAnalysis}.
In addition, there is substantial technical novelty in the way we bound the minimax program in \cref{sec:ProofOfMinimax},
where the statistical localization, as well as the definition 
of $\BeCom$, are leveraged explicitly. 

\paragraph{Setting up the proof}
In order to prove the theorems, we need to establish a high probability bound
on the estimation error, which is the value function difference at the initial state $\state_0$,
i.e., the quantity $\abs{\PE(\fhat)} = \abs{(\fstar - \fhat)(\state_0,\policy)}$.

The proof is based on the following key observation: 
since $\fhat$ minimizes the empirical loss $\MinimaxEmp$, 
we expect that we can bound its population value $\Minimax(\fhat)$.
Following the suggestion outlined in \cref{sec:OPC},
we factorize the squared prediction error as
\begin{align*}
	\PE(\fhat)^2
	= \frac{\PE(\fhat)^2}{\Minimax(\fhat)} \times \Minimax(\fhat)
	\leq \OPC \times \Minimax(\fhat).
\end{align*}
The off-policy cost coefficient $\OPC$ is defined in \cref{eqn:OPC},
and connects the prediction error to the population-based value of the minimax program.
In order to complete the proof, we need to bound $\OPC$ and $\Minimax(\fhat)$.
\paragraph{Bounding $\OPC$}
A variation of the simulation lemma \cite{kakade2003sample} allows us 
to upper bound the numerator in $\OPC$; it is proved in \cref{sec:WeakSimLem}.
\begin{lemma}[Weak Simulation Lemma]
	\label{lem:WeakSimLem}
	For any $\f \in \fClass$ we have the bound
	\begin{align*}
		\abs{\PE(\f)}  
		\leq 
		\frac{1}{1-\discount}  \norm{\f - \T\f}{\policy}.
	\end{align*} 
\end{lemma}
In addition, we can lower bound the denominator in $\OPC$ with simple algebra.
\begin{lemma}[Effect of $\BeCom$-incompleteness]
\label{lem:BeInc}
For any $\f \in \fClass$ we have the bound
	\begin{align*}
		\norm{\f - \T\f}{\Dist}^2 \leq \frac{1}{1-\BeCom}\Minimax(\f).
\end{align*}
\end{lemma}
The above lemma is where the definition of $\BeCom$-incompleteness is leveraged;
however, $\BeCom$-incompleteness also plays a role 
in determining the rate of convergence of the minimax program in
\cref{prop:Minimax}.
After putting together the pieces, we obtain
\begin{align*}
	\OPC 
	& \leq \sup_{\f \in \fClass} \frac{\PE(\f)^2}{\Minimax(\f)} \\
	& \leq \sup_{\f \in \fClass}
	\Big( \frac{1}{1-\discount} \Big)^2 \frac{1}{1-\BeCom}  
	\frac{\norm{\f - \T\f}{\policy}^2}{\norm{\f - \T\f}{\Dist}^2} \\
	& = \Big( \frac{1}{1-\discount} \Big)^2 \frac{1}{1-\BeCom}
	\Concentrability.
\end{align*}

\paragraph{Bounding $\Minimax(\fhat)$}
In order to conclude, we must establish 
a high probability rate of convergence for the population loss 
evaluated at the empirical minimizer $\fhat$.
Such rate of convergence depends on the function class $\fClass$.
More precisely, if the function class $\fClass$ has finite cardinality, 
the rate of convergence is 
$$\criticalradius^2 \simeq \crfc,$$
while for a general function class it must be such that 
any $\radius \geq \criticalradius$ satisfies \cref{eqn:RateOfConvergenceGeneral}.
\begin{proposition}[Rate of Minimax]
\label{prop:Minimax}
	With probability at least $1-\FailureProbability$
	\begin{align*}
		\Minimax(\fhat) \lesssim \criticalradius^2.
	\end{align*}
\end{proposition}
The proof of \cref{prop:Minimax} is in the appendix.
Combined with the bound on $\OPC$, the proof of 
\cref{thm:MinimaxErrorBoundFC,thm:MinimaxErrorBoundGeneral} is complete.

\subsection{Proof of \fullref{lem:WeakSimLem}}
\label{sec:WeakSimLem}
For a fixed function $\f\in\fClass$, 
the simulation lemma (e.g., \cite{kakade2003sample}) ensures
\begin{align*}
	\abs{\PE(\f)} 
	& = \abs{(\fstar - \f)(\state_0,\policy)} \\
	& = 
	\abs{\frac{1}{1-\discount}\E_{\psa \sim \dpi{\policy}} (\f -\T \f)\psa} \\
	& \leq
	\frac{1}{1-\discount} \E_{\psa \sim \dpi{\policy}} \sqrt{[( \f -\T \f )\psa]^2 }
\intertext{Using the Jensen's inequality we obtain the upper bound}
& \leq \frac{1}{1-\discount} 
\sqrt{\E_{\psa \sim \dpi{\policy}} [(\f - \T \f) \psa ]^2} \\
& = \frac{1}{1-\discount} 
\norm{\f -\T \f}{\policy}.
\end{align*}

\subsection{Proof of \fullref{lem:BeInc}}
We can write
\begin{align*}
	\Minimax(\f) 
	& =		
	\norm{\f - \T\f}{\Dist}^2
	-
	\norm{\gbest{\f} - \T\f}{\Dist}^2 \\
	& \geq 
	\norm{\f - \T\f}{\Dist}^2
	-
	\BeCom^2 \norm{\f - \T\f}{\Dist}^2 \\
	& =
	(1- \BeCom^2)\norm{\f - \T\f}{\Dist}^2 \\
	& =
	(1 - \BeCom)(1+\BeCom)\norm{\f - \T\f}{\Dist}^2 \\
	& \geq
	(1 - \BeCom)\norm{\f - \T\f}{\Dist}^2.
\end{align*}
The last inequality follows from the fact that $\BeCom \in [0,1]$.

\newpage
\section{Proof of \fullref{prop:Minimax}}
\label{sec:ProofOfMinimax}
We will show that the population and the empirical loss are related, 
i.e., that
$$
\Minimax(\fhat) \lesssim 2\MinimaxEmp(\fhat)
$$ 
with high probability.
Next, since $\fhat$ minimizes $\MinimaxEmp$, and realizability holds, 
we should have that $\MinimaxEmp(\fhat)$ is small, or more precisely that
with high probability
$$\MinimaxEmp(\fhat) \lesssim \criticalradius^2.$$
Together, they imply the statement.
In order to proceed we need to introduce more notation.

\subsection{Notation, Empirical Processes and Failure Events}
We need to show that the bad event
\begin{align}
	\label{eqn:BadEvent}
	\Minimax(\fhat) \gtrsim 2\criticalradius^2
\end{align}
occurs with probability at most $\FailureProbability$.
Since $\fhat$ is random, we establish 
uniform convergence results, i.e., 
statements that hold for many (possibly all) functions $\f \in \fClass$.
In order to do so,
we need to analyze the statistical fluctuations of the empirical process
associated to the cost function that defines the loss:
\begin{align*}
	\Xvar(\f) \defeq \Cost(\f, \f) - \Cost(\gbest{\f}, \f).
\end{align*}
This is a natural quantity to analyze, 
because its expectation
(which is computed with the help of \cref{lem:ExpectiCost})
is precisely the quantity that we wish to control
\begin{align*} 
	\P \Xvar(\f) 
	& =
	\E_{\psa \sim \Dist}
	\Big[ \E_{\reward \sim \RewardLaw\psa, \successorstate \sim \Transition{}\psa} \Xvar(\f) \Big] \\
	& = 
	\Loss(\f, \f) + \sigma(\f)^2 - \Loss(\gbest{\f}, \f) - \sigma(\f)^2 \\
	& = 
	\Loss(\f, \f) - \Loss(\gbest{\f}, \f) \\
	& =
	\Loss(\f, \f) - \inf_{\g\in\fClass}\Loss(\g, \f) \\
	& =
	\Minimax(\f),
\end{align*}
while its empirical average is upper bounded by the empirical loss that the agent minimizes
\begin{align*}
	\Pn \Xvar (\f)
	& =
	\frac{1}{\nSamples}\SumOverData{} 
	\Xvar(\f) \\
	& =
	\LossEmp(\f,\f) - \LossEmp(\gbest{\f}, \f) \\
	& \leq
	\LossEmp(\f,\f) - \inf_{\g \in \fClass}\LossEmp(\g, \f) \\
	& = \MinimaxEmp(\f).
\end{align*}

\subsubsection{Setting up the Failure Events}
As outlined, we need to establish that it is unlikely that $\MinimaxEmp(\fhat)$ is large
\begin{align}
\label{eqn:FailEvent_1}
\Pro(\FailEvent_1) \leq \FailureProbability/2
\qquad \text{where} \;
\FailEvent_1 \; : \quad
\MinimaxEmp(\fhat)  > \criticalradius^2.
\end{align}
When the failure event $\FailEvent_1$ does not occur,
we have $ \Pn\Xvar(\fhat) \leq \MinimaxEmp(\fhat) \leq \criticalradius^2$.
If we can claim $\Minimax(\fhat) = \P \Xvar(\fhat) \leq 2\Pn\Xvar(\fhat)$ 
then the proof would be complete.
Unfortunately, the latter claim is not true in general.
However, notice that if
$\Minimax(\fhat) = \P \Xvar(\fhat) \leq \criticalradius^2$
then we can already jump to the conclusion.
Therefore, it is sufficient (and more convenient) to show that 
it is unlikely that $\P \Xvar(\fhat)$ is large (i.e., $ > \criticalradius^2$)
and at the same time the deviation is large $\P \Xvar(\fhat) > 2\Pn\Xvar(\fhat)$:
\begin{align}
\label{eqn:FailEvent_2}
\Pro(\FailEvent_2)
\leq \FailureProbability/2
\qquad \text{where} \;
\FailEvent_2 \; : \quad 
\P \Xvar(\fhat) > 2\Pn\Xvar(\fhat)
\quad \text{and} \quad  
\P \Xvar(\fhat) > \criticalradius^2.
\end{align}
To recap:
when neither $\FailEvent_1$ nor $\FailEvent_2$ occur either we have
\begin{align*}
	\Minimax(\fhat) = \P\Xvar(\fhat) \leq \criticalradius^2
\end{align*}
or otherwise we have
\begin{align*}
	\Minimax(\fhat) = \P\Xvar(\fhat) \leq 2\Pn(\fhat) \leq 2\MinimaxEmp(\fhat) \leq 2\criticalradius^2,
\end{align*}
and the proof would be complete.
Consequently, the rest of the proof is devoted to showing
that $\FailEvent_1$ and $\FailEvent_2$ are unlikely to occur, namely 
the claims in \cref{eqn:FailEvent_1,eqn:FailEvent_2}.

\subsubsection{Relaxing the Failure Events}
In this section we define events that are easier to bound and that lead
to the stated result in \cref{eqn:FailEvent_1,eqn:FailEvent_2}. 

\paragraph{Relaxing the claim in \texorpdfstring{\cref{eqn:FailEvent_2}}{}}
	The difference $\P \Xvar(\f) - \Pn\Xvar(\f) $ is a concentration term.
	It is convenient to introduce the set of functions under consideration
	\begin{align*}
		\PartSetTD(\criticalradius) = \{ \f \in \fClass \; \mid \; \P\Xvar(\f) > \criticalradius^2 \}.
	\end{align*}
	To establish the claim in \cref{eqn:FailEvent_2} it is enough to establish that
	large deviations are unlikely for all functions with large expectation,
	i.e., that
	\begin{align}
		\label{eqn:FailEvent_1bis} 
		\text{$\exists \f \in \PartSetTD(\criticalradius)$ such that } \quad 
		\P \Xvar(\f) - \Pn\Xvar(\f) > \frac{1}{2}\P \Xvar(\f)
	\end{align} 
	can occur with probability at most $\FailureProbability/2$.
	
	\paragraph{Relaxing the claim in \texorpdfstring{\cref{eqn:FailEvent_1}}{}}
	In order to provide the required bound,
	we need to leverage the fact that $\fhat$ is minimizing $\MinimaxEmp(\f)$.
	\begin{align*}
		\MinimaxEmp(\fhat) 
		\leq
		\MinimaxEmp(\fstar) 
		=
		\LossEmp(\fstar,\fstar) - \LossEmp(\gbestemp{\fstar}, \fstar).
	\end{align*}
	The term to bound is the empirical (excess) risk of a realizable problem.
	For convenience, define the empirical process
	\begin{align*}
		\Yvar(\g) \defeq \Cost(\fstar,\fstar) - \Cost(\g,\fstar).
	\end{align*} 
	With the above definition we have
	\begin{align*}
		\Pn \Yvar(\g) 
		& = 
		\LossEmp(\fstar, \fstar) - \LossEmp(\g, \fstar), \\
		\P \Yvar(\g) 
		& = 
		\Loss(\fstar, \fstar) - \Loss(\g, \fstar) \leq 0
	\end{align*}
	To recap:
	if we can show that with probability $1-\FailureProbability/2$
	\begin{align}
	\label{eqn:CI2}
	\Pn \Yvar(\g)	
	\leq
	\frac{1}{2} \criticalradius^2
	\qquad \text{for all } \g \in \fClass
	\end{align}
	then under the same event we have the desired bound
	\begin{align*}
		\MinimaxEmp(\fhat) \leq \Pn \Yvar(\gbestemp{\fstar}) \leq \frac{1}{2} \criticalradius^2.
	\end{align*}
	
	\newpage
	\subsection{Concentration inequalities for finite classes}
	In this section we complete the proof for the special case where $\fClass$
	has finite cardinality.
	
	\subsubsection{Establishing \cref{eqn:FailEvent_1bis}}
	To complete the proof, we need to compute the threshold $\criticalradius$
	past which the event in \cref{eqn:FailEvent_1bis} becomes unlikely.
	
	The Bernstein's inequality (see e.g., \cite{wainwright2019high} for a reference), coupled with a union bound over each function 
	in $\PartSetTD(\criticalradius) \subseteq \fClass$ 
	ensures that the following event occurs with probability at most $\FailureProbability/2$
	\begin{align*}
		\text{$\exists  \f \in\PartSetTD(\criticalradius)$ such that } \quad  
		\P \Xvar(\f) - \Pn\Xvar(\f) 
		\gtrsim 
		\sqrt{\frac{\Var\Xvar(\f) \ln(\card{\fClass}/\FailureProbability)}{\nSamples} } 
		+ \frac{\ln(\card{\fClass}/\FailureProbability) }{\nSamples}.
	\end{align*}
	If we make the above right hand side larger
	then the event becomes even more unlikely.
	The term involving the variance can be upper bounded by upper bounding the variance
	\begin{align*}
		\Var\Xvar(\f) \leq \factor\P\Xvar(\f), 
		\qquad
		\text{where}
		\qquad
		\factor
		\lesssim
		\frac{1}{1-\BeCom},
	\end{align*}
	a result stated in \cref{lem:Xvariance}.
	If in addition the fast rate is dominated by the variance term,
	(we shall see in few lines that this is the case),
	namely if for all functions in $\PartSetTD(\criticalradius)$
	\begin{align}
	\label{eqn:AddReqFC}
		\frac{\ln(\card{\fClass}/\FailureProbability)}{\nSamples}
		\lesssim 
		\sqrt{\frac{\factor\P\Xvar(\f) \ln(\card{\fClass}/\FailureProbability)}{\nSamples} },
	\end{align}
	then we readily obtain the smaller (and more unlikely) event defined below
	\begin{align*}
	\text{$\exists \f \in\PartSetTD(\criticalradius)$ such that } \quad
		\P \Xvar(\f) - \Pn\Xvar(\f) 
		\gtrsim 
		\sqrt{\frac{\factor\P\Xvar(\f) \ln(\card{\fClass}/\FailureProbability)}{\nSamples} }.
	\end{align*}
	The fact that \cref{eqn:FailEvent_1bis} holds with probability at most $\FailureProbability{}/2$
	then would follow if its right hand side is even bigger than the right hand side 
	in the above display; such situation occurs if for all $\f \in \PartSetTD(\criticalradius)$
	\begin{align}
	\label{eqn:criticalInequalityFC}
		 \frac{1}{2}\P \Xvar(\f) 
		 \gtrsim
		 \sqrt{\frac{\factor\P\Xvar(\f) \ln(\card{\fClass}/\FailureProbability)}{\nSamples} }. 
	\end{align}
	Solving for $\P\Xvar(\f)$ gives the condition 
	\begin{align*}
		\P \Xvar(\f) \gtrsim \frac{\factor \ln(\card{\fClass}/\FailureProbability)}{\nSamples}.
		\end{align*}
	Such condition must be satisfied by all functions $\f\in \PartSetTD(\criticalradius)$,
	a fact that holds true by definition of $\PartSetTD(\criticalradius)$ as soon as $\criticalradius$ satisfies
	\begin{align}
	\label{eqn:CRFC}
		\criticalradius^2 
		\gtrsim 
		\frac{\factor \ln(\card{\fClass}/\FailureProbability)}{\nSamples}.
	\end{align}
	The value for $\criticalradius$ established by the above inequality
	ensures that any function $\f \in\PartSetTD(\criticalradius)$
	satisfies the bound in display in \cref{eqn:criticalInequalityFC}
	(recall the definition of $\PartSetTD(\criticalradius)$). 
	In addition, it also ensures that \cref{eqn:AddReqFC} is always satisfied, as promised 
	(observe that $\factor \geq 1$).
	
	To recap: we have computed the critical threshold $\criticalradius$
	past which \cref{eqn:FailEvent_1bis} occurs with vanishing probability, as desired.
	By doing so, we have also determined the rate of convergence $\criticalradius$
	of the minimax program, up to a constant.

	\subsubsection{Establishing \cref{eqn:CI2}}
		In this section we establish \cref{eqn:CI2},
		or equivalently that the following event has probability at most $\FailureProbability/2$:
		\begin{align}
			\label{eqn:CI2bis}
			\text{exists } \g \in \fClass \text{ such that} \quad
			\Pn \Yvar(\g)	
			>
			\frac{1}{2} \criticalradius^2.
		\end{align}
		We start from the inequality of Bernstein coupled with a union bound
		over each element of $\fClass$ to ensure that the following event 
		has probability at most $\FailureProbability/2$
		\begin{align*}
			\exists \g \in \fClass \qquad \text{such that} \qquad
			\Pn \Yvar(\g) - \P \Yvar(\g)  
			\gtrsim
			\sqrt{\frac{\Var\Yvar(\g) \ln(\card{\fClass} / \FailureProbability)}{\nSamples}}
			+
			\frac{\ln(\card{\fClass}/\FailureProbability)}{\nSamples}.
		\end{align*}
		If we make the right hand side in the above display any larger, 
		the event above becomes even more unlikely.
		We have the following bound on the variance (recall that $\P\Yvar(\g) \leq 0$),
		which we verify in \cref{lem:Yvar}
		\begin{align*}
			\Var\Yvar(\g) \lesssim - \P\Yvar(\g).
		\end{align*}
		We obtain the following (smaller) event
		\begin{align*}
			\exists \g \in \fClass \qquad \text{such that} \qquad
			\Pn \Yvar(\g) 
			-
			\P \Yvar(\g)
			\gtrsim 
			\sqrt{\frac{- \P\Yvar(\g) \ln(\card{\fClass} / \FailureProbability)}{\nSamples}}
			+
			\frac{\ln(\card{\fClass}/\FailureProbability)}{\nSamples}
			\end{align*}
			or equivalently 
			\begin{align*}
			\exists \g \in \fClass \qquad \text{such that} \qquad
			\Pn \Yvar(\g) 
			\geq
			\P \Yvar(\g)
			+
			c \sqrt{\frac{- \P\Yvar(\g) \ln(\card{\fClass} / \FailureProbability)}{\nSamples}}
			+
			\frac{\ln(\card{\fClass}/\FailureProbability)}{\nSamples}
		\end{align*}
		for some constant $c > 0$, a bound that
		holds with probability at most $\FailureProbability/2$.
		We would then be able to conclude that
		\cref{eqn:CI2bis} holds with probability at most $\FailureProbability/2$
		if its right hand side is always larger than the one in the above display, 
		namely when
		\begin{align*}
		\frac{1}{2}
			\criticalradius^2
			\geq
			\P \Yvar(\g) + 
			c\sqrt{\frac{- \P\Yvar(\g) \ln(\card{\fClass} / \FailureProbability)}{\nSamples}}
			+
			\frac{\ln(\card{\fClass}/\FailureProbability)}{\nSamples}
		\end{align*}
		The right hand side above is quadratic in $\sqrt{-\P\Yvar(\g)}$.
		Its maximum value\footnote{
		Notice that $\P \Yvar(\g)$ is negative while the square root term is positive;
		in particular, the right hand side is maximized when 
		$
			\P \Yvar(\g) \simeq \frac{\ln(\card{\fClass} / \FailureProbability)}{\nSamples}
		$.} is 
		\begin{align*}
			\frac{\ln(\card{\fClass} / \FailureProbability)}{\nSamples}
			\gtrsim
			\P \Yvar(\g) + 
			c\sqrt{\frac{- \P\Yvar(\g) \ln(\card{\fClass} / \FailureProbability)}{\nSamples}}
			+
			\frac{\ln(\card{\fClass}/\FailureProbability)}{\nSamples},
		\end{align*}
		and therefore it is sufficient that $\criticalradius$ satisfies the inequality
		\begin{align*}
			\criticalradius^2 
			\gtrsim
			\frac{\ln(\card{\fClass} / \FailureProbability)}{\nSamples}.
		\end{align*}
		In other words, we have determined the minimum value for $\criticalradius$
		past which \cref{eqn:CI2bis} becomes unlikely;
		furthermore, this requirement is already satisfied 
		by that presented in \cref{eqn:CRFC}.
	
	\newpage
	\subsection{Concentration inequalities for general functions}
	In this section we establish \cref{eqn:FailEvent_1bis,eqn:CI2}
	for general function classes. It is useful to define the following factor (up to a constant).
	\begin{align*}
		\factor
		\simeq
		\begin{cases}
			\frac{1}{1-\BeCom}, & \text{if $\fClass$ is non-convex} \\
			1, & \text{if $\fClass$ is convex}.
		\end{cases}
	\end{align*}
	
	\subsubsection{Establishing the claim in \cref{eqn:FailEvent_1bis}}
	In order to provide a bound to \cref{eqn:FailEvent_1bis},
	we need a suitable concentration inequality
	that can ensure fast rates by leveraging the variance of the process. 
	However, the analysis to follow deals with the worst-case variance 
	represented by the quantity
	$\sup_{\f \in \PartSetTD(\criticalradius)} \P\Xvar(\f)$
	which can be\footnote{
	We have 
	$\P\Xvar(\f) \leq \norm{\f - \T\f}{\Dist}^2 
	\leq 
	\norm{\f - \T\f}{\infty}^2
	\lesssim 1$.} of order one.
	In order to tightly connect the worst-case maximum variance to the actual value of $\P\Xvar(\f)$
	of the function responsible for violating the inequality in \cref{eqn:FailEvent_1bis},
	it is best to partition the set $\PartSetTD(\criticalradius)$ 
	\begin{align*}
		\PartSetTD(\criticalradius) = \cup_{m \in [M]} \PartSetTD_m
	\end{align*}
	according to the value of $\P\Xvar(\f)$, i.e.,
	using intervals that tightly bracket the possible values of $\P\Xvar(\f)$,
	as follows:
	$$ 
	\PartSetTD_m
	=
	\Big\{
	\f \in \PartSetTD(\criticalradius)
	\; \mid \; 
	\radius^2 < \P\Xvar(\f) \leq 2 \radius^2
	\Big\}, 
	\qquad \text{where } \radius^2 = 2^{m-1} \criticalradius^2.
	$$
	The partition starts at $m = 1 $ where $\radius = \criticalradius$
	and since (see footnote) $\P\Xvar(\f) \lesssim 1$,
	the partition can stop at $M \simeq \log_2(1/\criticalradius)$.
When $\f \in \PartSetTD_m$ we have $\P\Xvar(\f) \geq \radius^2$ 
and therefore we can create a larger event which is easier to bound
\begin{align*}
	\Big\{\exists \f \in \PartSetTD_m \; \mid 
	\P \Xvar(\f) - \Pn\Xvar(\f) > \frac{1}{2}\P \Xvar(\f) \Big\}
	\subseteq 
	\Big\{\exists \f \in \PartSetTD_m \; \mid 
	\P \Xvar(\f) - \Pn\Xvar(\f) > \frac{1}{2}\radius^2 \Big\} \defeq \Event_m.
\end{align*}
Let $\Event$ be the event in \cref{eqn:FailEvent_1bis};
using the above inclusion, we can claim
\begin{align*}
	\Event \subseteq \cup_{m \in [M]} \Event_m.
\end{align*}
At this point we can apply \cref{lem:TalagrandBracket};
rescaling $\FailureProbability$ coupled with the union bound
now gives a bound on the original event
\begin{align*}
	\Pro(\Event) 
	\leq 
	\sum_{m \in [M]} \Pro(\Event_m) 
	\leq 
	\FailureProbability/2.
\end{align*}
In order to apply \cref{lem:TalagrandBracket},
several conditions must be met. 
The bound on the variance is ensured by \cref{lem:Xvariance};
in addition, $\radius$ must satisfy the following two critical relations 
for appropriate constants and for all $m \in [M]$
\begin{align}
\label{eqn:ProblemDependentCriticalInequality}
	\E \sup_{\f \in \PartSetTD_m} 
	\Big\{ \P\Xvar(\f) - \Pn \Xvar(\f) \Big\}
	\lesssim \radius^2,  
	\qquad \text{and} \qquad
	(\factor + 1) \frac{\ln(1/( \FailureProbability\criticalradius ))}{\nSamples} 
	\lesssim \radius^2.
\end{align}
The condition on the left involves the Bellman operator $\T$
through $\Xvar(\f)$.
The requirement is relaxed in \cref{lem:RemovingBellman};
we obtain that it is sufficient that $\radius$ satisfies an inequality
that involves the following local Rademacher averages:
\begin{align}
\label{eqn:above}
	\CriticalCondition{\radius}
\end{align}
If both conditions admit a smallest positive solution $\criticalradius$
such that \cref{eqn:above} holds for all $\radius \geq \criticalradius$ 
then we can cover all cases $m \in [M]$ with the condition
$\radius \geq \criticalradius$ where $\criticalradius$ satisfies
\begin{align}
\label{eqn:CriticalCondition1}
	\CriticalCondition{\criticalradius}
\end{align}
Since $\radius \geq \criticalradius$, 
when the inequalities in \cref{eqn:CriticalCondition1} are satisfied,
\cref{eqn:above} is automatically satisfied as well.

\subsubsection{Establishing the claim in \cref{eqn:CI2}}

	Since $\P \Yvar(\g)  \leq 0$, it is sufficient to claim 
	that we are unlikely to witness 
	large deviations such as the one below\footnote{Notice that \cref{eqn:FailEvent2} is a bit weaker than \cref{eqn:FailEvent_1bis}
	due to the additional $\criticalradius$ term on the right hand side;
	this is due to the fact that we must also consider the small variance regime 
	$ - \P \Yvar(\g) \leq \criticalradius^2$ in this section.}:
	\begin{align}
	\label{eqn:FailEvent2}
	\exists \g \in \fClass \; \text{such that} \; 
	\Pn\Yvar(\g) - \P \Yvar(\g) 	
	>
	- \frac{1}{2}\P \Yvar(\g) 
	+
	\frac{1}{2}
	\criticalradius^2.
	\end{align}
	Let $\Event$ be the above event;
	we show that $\Event$ can occur with probability at most $\FailureProbability/2$.
	In the complement event, \cref{eqn:CI2} must hold.
	
We construct a family of sets $\{ \Event_m \}$ such that
\begin{align*}
	\EventY \subseteq \cup_{m \in \{0,1,2,\dots,M \}} \EventY_m
\end{align*}
where each event $\EventY_m$ is described in the analysis to follow.
\paragraph{Small variance event}
Let us consider the set of functions with small variance
\begin{align*}
	\fClass_0 = \{\g \in \fClass \mid 0 \leq - \P\Yvar(\g) \leq \frac{1}{2}\criticalradius^2 \}.
\end{align*}
The associated event is 
\begin{align*} 
	& \Bigg\{
	\exists \g \in \fClass_0 \; \text{such that} \;
	\Pn \Yvar(\g) - \P \Yvar(\g)
	> 
	- \frac{1}{2}\P\Yvar(\g) + \frac{1}{2}\criticalradius^2 
	\Bigg\} \\
	\subseteq &
	\Bigg\{
	\exists \g \in \fClass_0 \; \text{such that} \;
	\Pn \Yvar(\g) - \P \Yvar(\g) 
	> 
	\frac{1}{2}\criticalradius^2 
	\Bigg\} \\
	\defeq & \PartEventY_0.
	\end{align*}

\paragraph{Large variance events}
Consider the following partitioning to control the variance of the empirical process
\begin{align*}
	\fClass_m = \{ \radius^2 < - \P\Yvar(\g) \leq 2 \radius^2 \},
	\qquad \text{where }  \radius^2 \defeq 2^{m-2} \criticalradius^2 
	\geq \frac{1}{2} \criticalradius^2,
	\quad \text{for } m = 1,2,\dots,M.
\end{align*}
The partition stops at $M \simeq \ln (1/\criticalradius)$ as $- \P\Yvar(\g) \lesssim 1$.
The associated events are 
\begin{align*}
	& \hphantom{=} 
	\Bigg\{ 
	\exists \g \in \fClass_m 
	\; \text{such that} \;
	\Pn \Yvar(\g) - \P \Yvar(\g) 
	\geq - \frac{1}{2}\P\Yvar(\g) + \frac{1}{2} \criticalradius^2
	\Bigg\} \\
	& = 
	\Bigg\{ 
	\exists \g \in \fClass_m 
	\; \text{such that} \;
	\Pn \Yvar(\g) - \P \Yvar(\g) 
	\geq - \frac{1}{2}\P\Yvar(\g)
	\Bigg\} \\
	& \subseteq
	\Bigg\{ 
	\exists \g \in \fClass_m 
	\; \text{such that} \;
	\Pn \Yvar(\g) - \P \Yvar(\g)
	\geq \frac{1}{2} \radius^2
	\Bigg\} \\
	& \defeq \PartEventY_m.
\end{align*}

\paragraph{Putting together the pieces}
After rescaling $\FailureProbability$ to become $\FailureProbability / (2(M+1))$
and using the union bound we can finally
apply \cref{lem:TalagrandBracket} 
to bound the event in \cref{eqn:FailEvent2}
\begin{align*}
	\Pro(\EventY) \leq \sum_{m} \Pro(\PartEventY_m) \leq \FailureProbability/2,
\end{align*}
In order to apply \cref{lem:TalagrandBracket},
we need to verify the assumptions in the statement of the lemma.

We have the following variance calculation reported in \cref{lem:Yvar}
\begin{align}
\label{eqn:VarCalcII}
\Var[\Yvar(\g)] \lesssim 
- \P\Yvar(\g),
\qquad 
\text{for all $\g \in \fClass$.}
\end{align}
By the symmetry of $\fClass$, every time \cref{lem:TalagrandBracket} is invoked,
for every $m = 1,2, \dots,M$ the associated value for $\radius$ must satisfy
\begin{align*}
	\E \sup_{\g \in \fClass_m} 
	\Big\{ \P\Yvar(\g) - \Pn \Yvar(\g) \Big\}
	\lesssim \radius^2,  
	\qquad \text{and} \qquad
	\frac{\ln(1/( \FailureProbability \criticalradius))}{\nSamples} \lesssim \radius^2.
\end{align*}
The condition on the right is already in the final form;
the one on the left involves the Bellman operator $\T$
through $\Yvar(\f)$.
In order to obtain a bound that only depends on the class $\fClass$,
the requirement is relaxed in \cref{lem:RemovingBellmanFromY}.
After the relaxation, we obtain that 
it is sufficient that $\radius$ satisfies the inequalities
\begin{align*}
	\Rademacher_\nSamples \big [ \big(\fClass - \fstar \big) 
	(\radius)\big ] 
	\lesssim \radius^2
	\qquad \text{and} \qquad
	\frac{\ln(1/(\FailureProbability\criticalradius))}{\nSamples} \lesssim \radius^2.
\end{align*}

\newpage

\subsubsection{Talagrand's Bound}
In this section we assume that the prerequisites for using Talagrand are met in order to avoid measurability issues,
and bound the supremum of an empirical process.
Let $\Wvar$ be a random variable on a certain probability space. 
For a given function class $\hClass$, define the supremum of the empirical process
\begin{align*}
	\Zvar
	& = 
	\sup_{\h \in \hClass} 
	\Big\{ \P\h(\Wvar) - \Pn \h(\Wvar) \Big\}.
\end{align*}
We use Talagrand's bound \cite{talagrand1996new} to derive a
tail bound to $\Zvar$ when the variance of the process is tightly bracketed.
(Here $ \Mag \geq 1$).

\begin{lemma}[Talagrand's Bound with Bracketed Variance]
\label{lem:TalagrandBracket}
The event
\begin{align*}
	\TEvent \; : \quad 
	\Zvar > \frac{1}{2}\radius^2
\end{align*}
occurs with probability at most $\FailureProbability$
if the following conditions are satisfied for appropriate universal constants
\begin{align}
\label{eqn:TalCondi}
	\Var [\h(\Wvar)] \leq \Mag \P \h(\Wvar) \leq 2\Mag\radius^2, 
	\qquad \text{and} \qquad
	\E\Zvar \lesssim \radius^2,  
	\qquad \text{and} \qquad
	(\Mag + 1) \frac{\lTerm}{\nSamples} \lesssim \radius^2.
\end{align}
\end{lemma}
The strategy is to create a more `natural' tail event $\TailEvent$ 
associated to a variance-based concentration inequality.
The concentration inequality will ensure
that $\Pro(\TailEvent) \leq \FailureProbability$.
Then we show that the event $\TEvent$ is contained in $\TailEvent$,
ensuring $\Pro(\TEvent) \leq \Pro(\TailEvent) \leq \FailureProbability$.

Talagrand's bound (see Thm 3.27 and Eq. 3.85 in the book \cite{wainwright2019high}) applies to the tail event
\begin{align}
\label{eqn:Talagrand}
	\TailEvent \: : \quad
	\Zvar
	& \gtrsim
	\E \Zvar
	+
	\sqrt{\frac{\sigma^2 \lTerm}{\nSamples}}
	+
	\frac{\lTerm}{\nSamples}.
\end{align}
where the variance-proxy of the process is 
\begin{align*}
\sigma^2 
& \defeq
\sup_{\h \in \hClass} 
\P \Big\{ 
\h(\Wvar) 
- 
\P \h(\Wvar) 
\Big\}^2
+
2\P \Zvar.
\end{align*}
It ensures that such large deviations
have small probability of occurring
\begin{align}
\label{eqn:TailBound}
	\Pro(\TailEvent) \leq \FailureProbability.
\end{align}
We now proceed to showing that 
$$\TEvent 
= 
\Big\{ \Zvar \geq \frac{1}{2}\radius^2 \Big\}
\subseteq 
\Big\{ \Zvar \geq 
	\E \Zvar
	+
	\sqrt{\frac{\sigma^2 \lTerm}{\nSamples}}
	+
	\frac{\lTerm}{\nSamples} 
	\Big\} 
	=
	\TailEvent, $$
which allows us to conclude.
In order to show the inclusion, we need to ensure that
\begin{align*}
\frac{1}{2}\radius^2
\geq
\E \Zvar
+
\sqrt{\frac{\sigma^2 \lTerm}{\nSamples}}
+
\frac{\lTerm}{\nSamples}
\defeq \rhs.
\end{align*}
We start from the above rhs and upper bound it
until we obtain $\radius^2$.
By combining the bound on the variance with the one on the expectation we obtain
\begin{align*}
\sigma^2 
& =
\sup_{\h \in \hClass} 
\P \Big\{ 
\h(\Wvar) 
- 
\P \h(\Wvar) 
\Big\}^2
+
2\P \Zvar \\
& =
\sup_{\h \in \hClass} 
\Var [ \h(\Wvar) ]
+
2\P \Zvar \\
& \leq
\sup_{\h \in \hClass} 
\Mag \P \h(\Wvar) 
+
2\P \Zvar \\
& \leq
2\Mag\radius^2
+
2\P \Zvar \\
& \lesssim
(\Mag+1)\radius^2,
\end{align*}
where the last step used \cref{eqn:TalCondi}.
This implies the upper bound
\begin{align*}
	\rhs
	& \leq
	\E \Zvar
	+
	\radius \sqrt{(\Mag+1) \frac{\lTerm}{\nSamples} }
	+
	\frac{\lTerm}{\nSamples} \\
	& \leq
	\frac{1}{2}\radius^2,
\end{align*}
where the last step used again \cref{eqn:TalCondi} for appropriately tuned numerical constants.
Therefore, we have shown the inclusion 
$\TEvent \subseteq \TailEvent$; 
combined with the tail bound \cref{eqn:TailBound} we obtain
\begin{align*}
	\Pro(\TEvent) \leq \Pro(\TailEvent) \leq \FailureProbability,
\end{align*}
as claimed.

\newpage
\subsubsection{Simplifying the Rademacher Complexities}
\begin{lemma}[Rademacher Complexities for the $X$ process]
\label{lem:RemovingBellman}
We have the relation
	\begin{align*}
	\E
	\sup_{\P \Uvar(\f) \leq 2 \radius^2 } 
	\Big\{ \P\Uvar(\f) - \Pn \Uvar(\f) \Big\}
	\lesssim
	\E \sup_{\P \Uvar(\f) \leq 2\radius^2}  
	\frac{1}{\nSamples} \SumOverSamples
	\RadVar_\iSample[ \Cost(\f,\f) - \Cost(\gbest{\f},\f )].
	\end{align*}
\end{lemma}

First notice that we have
\begin{align*}
	\norm{\f - \gbest{\f}}{\Dist} 
	\leq 
	\norm{\f - \T\f }{\Dist} 
	+ 
	\norm{ \gbest{\f} - \T\f}{\Dist} 
	\leq 2\norm{ \f - \T\f}{\Dist},
\end{align*}
and so using \cref{lem:BeInc}
\begin{align*}
	\norm{\f - \gbest{\f}}{\Dist}^2 \lesssim \frac{1}{1-\BeCom} \P \Uvar(\f).
\end{align*}
If instead $\fClass$ is convex, the Pythagoras' theorem ensures
\begin{align}
	\norm{\f - \gbest{\f}}{\Dist}^2 \leq  \P \Uvar(\f),
\end{align}
We handle both cases with 
\begin{align}
\label{eqn:LimitedExpansion}
	\norm{\f - \gbest{\f}}{\Dist}^2 \lesssim \factor \P \Uvar(\f).
\end{align}
It is useful to rewrite the left hand side in the statement of the lemma
in a better form first by using a symmetrization argument,
for which it is temporarily useful to emphasize the dependence on the random variable
$\Gvar = \sars{}$ 
\begin{align*}
	& = 
	\E \sup_{\P \Uvar(\f) \leq 2\radius^2} 
	\Big\{ \P\Uvar(\f)(\Gvar) - \Pn \Uvar(\f)(\Gvar_\iSample) \Big\} \\
	& =
	\E \sup_{\P \Uvar(\f) \leq 2\radius^2} 
	\frac{1}{\nSamples} \SumOverSamples
	\Big\{ \P\Uvar(\f)(\Gvar) - \Uvar(\f)(\Gvar_\iSample) \Big\} \\
	\intertext{Upon defining i.i.d. random variables $\widetilde \Gvar_i$ we can write}
	& =
	\E \sup_{\P \Uvar(\f) \leq 2\radius^2}  
	\frac{1}{\nSamples} \SumOverSamples
	\Big\{ \P\Uvar(\f)(\widetilde \Gvar_i) - \Uvar(\f)(\Gvar_\iSample) \Big\}.
	\intertext{Using Jensen's inequality we obtain}
	& \leq
	\E \sup_{\P \Uvar(\f) \leq 2\radius^2}  
	\frac{1}{\nSamples} \SumOverSamples
	\Big\{ \Uvar(\f)(\widetilde \Gvar_i) - \Uvar(\f)(\Gvar_\iSample) \Big\}. 
	\intertext{Now introduce the Rademacher random variables $\RadVar_\iSample$}
	& =
	\E \sup_{\P \Uvar(\f) \leq 2\radius^2}  
	\frac{1}{\nSamples} \SumOverSamples
	\RadVar_\iSample \Big\{ \Uvar(\f)(\widetilde \Gvar_i) - \Uvar(\f)(\Gvar_\iSample) \Big\} \\
	& \leq
	2 \E \sup_{\P \Uvar(\f) \leq 2\radius^2}  
	\frac{1}{\nSamples} \SumOverSamples
	\RadVar_\iSample \Big\{ \Uvar(\f)(\Gvar_i) \Big\}  \\
	& =
	2 \E \sup_{\P \Uvar(\f) \leq 2\radius^2}  
	\frac{1}{\nSamples} \SumOverSamples
	\RadVar_\iSample[ \Cost(\f,\f) - \Cost(\gbest{\f},\f )].
	\end{align*}
	The above argument is standard (see the textbook \cite{wainwright2019high}, chapter 4).

\begin{lemma}[Rademacher Complexities for the $Y$ process]
\label{lem:RemovingBellmanFromY}
We have the relation
	\begin{align*}
	\E
	\sup_{- \P \Yvar(\g) \leq 2 \radius^2 } 
	\Big\{ \P\Yvar(\g) - \Pn \Yvar(\g) \Big\}
	\leq
	\Rademacher_\nSamples\{ \Increment \in \fClass - \fClass, \mid \norm{\Increment}{\Dist}^2 
	\leq 2  \radius^2 \}.
	\end{align*}
\end{lemma}
Notice that we have
\begin{align}
\label{eqn:LimitedExpansionYprocess}
\norm{\g - \fstar}{\Dist}^2 = - \P\Yvar(\g).
\end{align}
by definition of $\Yvar$.
It is useful to rewrite the left hand side in the statement of the lemma
in a better form first by using a symmetrization argument;
this step is analogous to that in \cref{lem:RemovingBellman}, 
and thus here we report only the final bound:
\begin{align*}
	\E \sup_{- \P \Yvar(\g) \leq 2\radius^2} 
	\Big\{ \P\Yvar(\g) - \Pn \Yvar(\g) \Big\}
	& \leq
	2 \E \sup_{ - \P \Yvar(\g) \leq 2\radius^2}  
	\frac{1}{\nSamples} \SumOverSamples
	\RadVar_\iSample[ \Cost(\fstar ,\fstar) - \Cost(\g,\fstar )].
	\end{align*}
	Using \cref{eqn:LimitedExpansionYprocess} we obtain 
	\begin{align*}
	& =
	2 \E \sup_{\norm{\g - \fstar}{\Dist}^2\leq 2\radius^2} 
	\frac{1}{\nSamples} \SumOverSamples
	\RadVar_\iSample[ \Cost(\fstar,\fstar) - \Cost(\g,\fstar )].
	\intertext{
	Since 
	$$
	\abs{\Cost(\fstar ,\fstar) - \Cost(\g,\fstar )}
	= 
	\abs{(\fstar - \Backup\fstar)^2 - (\g - \Backup\fstar)^2}
	=
	\abs{
	(\fstar - \Backup\fstar - \g + \Backup\fstar)
	(\fstar - \Backup\fstar + \g - \Backup\fstar)
	}
	\lesssim 
	\abs{\fstar - \g},$$
	the Talagrand's contraction principle (see \cite{talagrand1996new} or Thm A.6 in \cite{bartlett2005local}) ensures}
	& \leq
	2 \E \sup_{\norm{\g - \fstar}{\Dist}^2\leq 2\radius^2} 
	\bigabs{ \frac{1}{\nSamples} \SumOverSamples
	\RadVar_\iSample (\g - \fstar)},	
\end{align*}
Thus we can re-write the above as
\begin{align*}
	2\E \sup_{\Increment  \in \fClass - \fstar, \; \norm{\Increment}{\Dist}^2\leq 2 \radius^2} 
	\bigabs{
	\frac{1}{\nSamples} \SumOverSamples
	\RadVar_\iSample{\Increment}
	},	
\end{align*}
which is the Rademacher complexity of the set 
\begin{align*}
	(\fClass - \fstar)\big( 2 \radius \big) \defeq
	\Big\{\Increment \in \fClass - \fstar, \; 
	\norm{\Increment}{\Dist}^2 \leq 2 \radius^2 \Big\}.
\end{align*}

\newpage
\section{Technical Results}
\label{sec:TechRes}

\subsection{Proof of \fullref{prop:iLin}}
\label{sec:iLin}	
	Since $\fClass$ is linear, the projector $\Projector$ onto $\fClass$ is a linear map.
	We can write
	\begin{align*}
		\f - \T\f 
		& = \f - \T\f - (\fstar - \T\fstar) \\
		& = (\f - \fstar) - \discount\Transition{}(\f - \fstar) \\
		& = (\IdentityOperator - \discount\Transition{})(\f - \fstar)
	\end{align*}
	and
	\begin{align*}
		\gbest{\f} - \T\f 
		& = \Projector\T\f - \T\f \\
		& = \Projector\T\f - \T\f - (\underbrace{\Projector\T\fstar}_{\fstar} - \T\fstar) \\
		& = \discount\Projector\Transition{}(\f - \fstar) - \discount\Transition{}(\f - \fstar) \\
		& = \discount( \Projector - \IdentityOperator) \Transition{}(\f - \fstar).
	\end{align*}
	Notice that $A = \discount( \Projector - \IdentityOperator) \Transition{}$ and
	$B = (\IdentityOperator - \discount\Transition{})$ 
	are both linear operators.
	If we denote with $\Increment = \f - \fstar$ the increments, we have
	\begin{align}
	\label{eqn:Inom}
		\Incompleteness(\radius) = 
		\sup_{\Increment \in \fClassLin, \; \norm{B\Increment}{\Dist} \leq \radius} 
		\norm{A\Increment}{\Dist}.
	\end{align}
	Fix $\radius > 0$ and let $\BeCom$ satisfy $\Incompleteness(\radius) = \BeCom \radius$
	for that specific value of $\radius$.
	Now, consider any other radius $\radius' > \radius$;
	it must be representable as $\radius' = c\radius$ for some constant $c > 1$.
	Then the function $\Increment' = c \Increment \in \fClassLin$ is feasible for the program below
	if $\Increment$ is feasible for the one in \cref{eqn:Inom}
		\begin{align}
		\label{eqn:Ilarge}
		\Incompleteness(c\radius) = 
		\sup_{\Increment \in \fClassLin, \; \norm{B\Increment}{\Dist} \leq c\radius} 
		\norm{A\Increment}{\Dist}.
	\end{align}
	This implies $\Incompleteness(c\radius) \geq \BeCom c\radius$.
	Now assume that the inequality is strict to derive a contradiction.
	That is, assume $\Incompleteness(c\radius) > \BeCom c\radius$
	and let $\Delta'$ be a maximizer of \cref{eqn:Ilarge}.
	Then the function $\Delta = \Delta'/c$ is feasible for 	\cref{eqn:Inom}
	and it gives $\Incompleteness(\radius) > \BeCom \radius$, contradiction,
	because we assumed $\Incompleteness(\radius) = \BeCom \radius$.
	Therefore we must have $\Incompleteness(\radius') 
	=  \Incompleteness(c\radius) = \BeCom c\radius = \BeCom \radius' $
	for any $\radius' > \radius$.
	Since $\radius$ is arbitrary, and $\Incompleteness(0) = 0$ 
	follows from \cref{prop:iBehav},
	the proof is complete.

\begin{lemma}[Expectation of the Single Cost]
\label{lem:ExpectiCost}
	$$
	\E \Cost(\g,\f)
	=
	\Loss{}(\g,\f) 
	+ 
	\E_{(\state,\action) \sim \Dist} 
	\Var_{ \reward \sim \RewardLaw\psa, \; \successorstate \sim \Transition{}(\state,\action)}
	\Big[ (\Backup\f)\pobs \Big].
	$$
\end{lemma}
\begin{proof}
Recall the definition of Bellman backup $\Backup$.
By some algebra steps we have
\begin{align*}
	\E [ \Cost(\g,\f) ]
	& =
	\E \Big(\g\psa - \reward - \discount \f(\successorstate, \policy) \Big)^2 \\
	& =
	\E \Big(\g\psa - (\T \f)\psa + (\T \f)\psa  - 
	\underbrace{ [\reward + \discount \f(\successorstate, \policy)] }_{= (\Backup \f) \pobs } \Big)^2 \\
	& =
	\E \Bigg\{ \Big(\g\psa - (\T \f)\psa \Big)^2 \\
	& + 2 \Big(\g\psa - (\T \f)\psa \Big) \Big( (\T\f)\psa - \Backup \f\pobs \Big) \\
	& + \Big( (\T \f)\psa - \Backup \f\pobs \Big)^2 \Bigg\} \\
	& =
	\Loss(\g, \f)  \\
	& +
	2 \E_{(\state,\action) \sim \Dist} \Big[ \Big(\g\psa - (\T \f)\psa \Big) 
	\underbrace{ \E_{ \reward \sim \RewardLaw\psa, \; \successorstate \sim \Transition{\policy}(\state,\action)} 
	\Big( (\T\f)\psa -  \Backup \f\pobs \Big)}_{= 0} \Big] \\
	& + 
	\E_{(\state,\action) \sim \Dist} 
	\E_{ \reward \sim \RewardLaw\psa, \; \successorstate \sim \Transition{}(\state,\action)}
	\Big[(\T\f)\psa - \Backup \f\pobs \Big]^2 \\
	& =
	\Loss(\g, \f)  
	+ 
	\E_{(\state,\action) \sim \Dist} 
	\Var_{ \reward \sim \RewardLaw\psa, \; \successorstate \sim \Transition{}(\state,\action)}
	\Big[\Backup \f\pobs \Big] 
\end{align*}
\end{proof}

\subsection{Variance Bounds}
\begin{lemma}[Variance of the $X$-process]
	\label{lem:Xvariance}
	We have the following bound on the variance
	\begin{align*}
		\Var[\Xvar(\f)] 
		& \lesssim
		\frac{1}{1-\BeCom} \P\Xvar(\f).
		\intertext{In addition, when $\fClass$ is convex then we have the tighter inequality}
		\Var[\Xvar(\f)] 
		& \lesssim 
		\P\Xvar(\f).
	\end{align*}
\end{lemma}
\begin{proof}
\begin{align*}
\Var[\Xvar(\f)] 
& =
\Var \big[ \Cost(\f,\f) - \Cost(\gbest{\f},\f ) \big]  \\
& \leq
\E \big[ \Cost(\f,\f) - \Cost(\gbest{\f},\f ) \big]^2 \\
& =
\E \Big[ (\f - \gbest{\f})^2(\f - \Backup\f + \gbest{\f} - \Backup\f)^2 \Big] \\
& \lesssim
\E ( \f - \gbest{\f})^2 \\
& = 
\norm{\f - \gbest{\f}}{\Dist}^2.
\end{align*}
When $\fClass$ is convex Pythagoras' theorem ensures
\begin{align*}
	\norm{\f - \gbest{\f}}{\Dist}^2 
	& \leq
	\norm{\f - \T\f}{\Dist}^2 - \norm{\gbest{\f} - \T\f}{\Dist}^2 \\
	& =
	\P\Xvar(\f).
\end{align*}
Otherwise, for arbitrary $\fClass$ we have the bound
\begin{align*}
	\norm{\f - \gbest{\f}}{\Dist} 
	& =
	\norm{\f - \T\f + \T\f - \gbest{\f}}{\Dist} \\
	& \leq
	\norm{\f - \T\f}{\Dist} + \norm{\gbest{\f} - \T\f}{\Dist} \\
	& \leq
	2 \norm{\f - \T\f}{\Dist}. 
\end{align*}
Coupled with \cref{lem:BeInc},
we obtain the bound
	\begin{align*}
	\norm{\f - \gbest{\f}}{\Dist}^2 
	\lesssim
	\norm{\f - \T\f}{\Dist}^2
	\lesssim
	\frac{1}{1-\BeCom}\P\Xvar(\f).
\end{align*}
\end{proof}

\begin{lemma}[Variance of the $\Yvar$-process]
\label{lem:Yvar} 
For any $\g \in \fClass$ we have the bound
$$\Var[\Yvar(\g)]  \leq - \P\Yvar(\g).$$
\end{lemma}
\begin{proof}
\begin{align*}
\Var[\Yvar(\g)] 
& =
\Var \big[ \Cost(\fstar,\fstar) - \Cost(\g,\fstar ) \big]  \\
& \leq
\E \big[ \Cost(\fstar,\fstar) - \Cost(\g,\fstar ) \big]^2 \\
& =
\E \Big[ (\fstar - \g)^2(\fstar - \Backup\fstar + \g - \Backup\fstar)^2 \Big] \\
& \lesssim
\E ( \fstar - \g)^2 \\
& = 
\norm{\fstar - \g}{\Dist}^2 \\
& = 
\norm{\T\fstar - \g}{\Dist}^2 \\
& = 
\Loss(\g,\fstar) \\
& =
\Loss(\g,\fstar) - \underbrace{\Loss(\fstar,\fstar)}_{= 0} \\
& =
- \P\Yvar(\g).
\end{align*}
\end{proof}

\end{document}